\documentclass{article}

\PassOptionsToPackage{square,numbers}{natbib}
\usepackage[final]{neurips_2021}

\usepackage{graphicx}
\usepackage[utf8]{inputenc} % allow utf-8 input
\usepackage[T1]{fontenc}    % use 8-bit T1 fonts
\usepackage{hyperref}       % hyperlinks
\usepackage{url}            % simple URL typesetting
\usepackage{booktabs}       % professional-quality tables
\usepackage{amsfonts}       % blackboard math symbols
\usepackage{nicefrac}       % compact symbols for 1/2, etc.
\usepackage{microtype}      % microtypography
\usepackage{xcolor}         % colors
\newcommand{\RR}{I\!\!R} %real numbers
\newcommand{\CC}{\mathcal{C}} %complex numbers
\usepackage{bbm}
\usepackage{algorithm}
\usepackage{algpseudocode}
\usepackage{enumerate}
\usepackage{booktabs}
\usepackage{pgfplots}
\usepackage{tikz-3dplot}

\usepackage{amssymb}
\usepackage{mathtools}
\usepackage{amsthm}
\usepackage{amsmath,amsfonts}
\usepackage{dsfont}

\tdplotsetmaincoords{60}{115}
\pgfplotsset{compat=newest}

\newcommand{\R}{\mathbb{R}}

\newtheorem{theorem}{Theorem}[section]
\newtheorem{proposition}[theorem]{Proposition}

\newtheorem{definition}[theorem]{Definition}

\newcommand{\norm}[1]{\left\vert #1 \right\vert}
\newcommand{\Norm}[1]{\left\Vert #1 \right\Vert}
\newlength{\arrow}
\usepackage{todonotes}
\usepackage{amsmath}

\title{Persistent Classification: Understanding Adversarial Attacks by Studying Decision Boundary Dynamics}

\author{%
  Brian Bell\\
  Department of Mathematics, \\
  University of Arizona \\
  Los Alamos National Laboratory\\
  \texttt{bwbell@lanl.gov} \\
    \And
  Michael Geyer\\
  Department of Computer Science, \\
  University of Texas San Antonio \\
  Los Alamos National Laboratory\\
  \texttt{mgeyer@lanl.gov}\\ 
    \And
  David Glickenstein\\
  Department of Mathematics, \\
  University of Arizona\\
  \And
  Keaton Hamm\\
  Division of Data Science, \\
  University of Texas at Arlington\\
  \And 
  Carlos Scheidegger\\
  Department of Computer Science, \\
  University of Arizona\\
  \And
  Amanda Fernandez\\
  Department of Computer Science \\
  University of Texas San Antonio\\
  \And
  Juston Moore\\
  Los Alamos National Laboratory\\
}

\begin{document}

\maketitle 

\begin{abstract}

There are a number of hypotheses underlying the existence of adversarial examples for classification problems. 
These include the high-dimensionality of the data, high codimension in the ambient space of the data manifolds of interest, and that the structure of machine learning models may encourage classifiers to develop decision boundaries close to data points.
% Some hypotheses underlying the existence of adversarial examples for classification problems are the high-dimensionality of the data, high codimension in the ambient space of the data manifolds of interest, and/or that the structure of machine learning models may encourage classifiers to develop decision boundaries close to data points. 
This article proposes a new framework for studying adversarial examples that does not depend directly on the distance to the decision boundary. 
Similarly to the smoothed classifier literature, we define a (natural or adversarial) data point to be $(\gamma,\sigma)$-stable if the probability of the same classification is at least $\gamma$ for points sampled in a Gaussian neighborhood of the point with a given standard deviation $\sigma$. 
We focus on studying the differences between persistence metrics along interpolants of natural and adversarial points.
% In order to provide a statistic for the stability of the classification of a point, we define $\gamma$-persistence of a point to be the largest standard deviation for which the point is $(\gamma,\sigma)$-stable; this determines how large a standard deviation one can choose so that Gaussian samples about the point tend to be classified the same way as the point itself. 
We show that adversarial examples have significantly lower persistence than natural examples for large neural networks in the context of the MNIST and ImageNet datasets. 
We connect this lack of persistence with decision boundary geometry by measuring angles of interpolants with respect to decision boundaries.
Finally, we connect this approach with robustness by developing a manifold alignment gradient metric and demonstrating the increase in robustness that can be achieved when training with the addition of this metric.

%  By considering $\gamma$ of varying sizes, we can relax the usual stability notion of requiring points to be sufficiently far away from the decision boundary. Changing the standard deviation leverages the fact that the Gaussian in high dimensions is concentrated near relatively narrow annuli to sense behavior at varying distances from the initial point. Together, we get a generalized notion of distance to the decision boundary that also captures oblique structure of the decision boundary.

% {\color{red}
% Title Ideas:
% \begin{enumerate}
%     \item Persistent neighborhoods of adversarial examples
%     \item (On) Stability of neighborhoods of adversarial examples
%     \item What do neighborhoods of adversarial examples look like?
%     \item Measuring stability of neighborhoods of adversarial examples
%     \item Characterizing stability of neighborhoods of adversarial examples
%     \item $(\gamma,\sigma)$-stability and persistence: Characterizing neighborhoods of adversarial examples
%     \item A new approach to stability of adversarial examples via persistence
%     \item Persistent Classification, A new Approach to Stability of Data under Learning Models
%     \item Persistent Classification: A new Approach to Stability of Data and Adversarial Examples
%     \item Persistent Classification: A new Approach to Stability of Data and adversarial examples
% \end{enumerate}
% }
\end{abstract}

The idea of expanding Gaussian perturbations around a data point has been studied previously as smoothed classifiers.
Whereas ~\citet{roth19aodds} consider adding various types of noise to a given point and ~\citet{hosseini2019odds} consider small Gaussian perturbations of $x$ sampled from $N(x,\varepsilon^2 I)$ for small $\varepsilon$, %\todo{[K]: Should we go ahead and just be clear here and say $N(x,\varepsilon^2 I)$?} %noise perturbations of $\varepsilon$ times a vector with normal distribution $N(0,I)$ with a small selection of choices for $\varepsilon.$
we specifically focus on %Our choice to consider varying standard deviations is quite similar to considering varying choices of $\varepsilon$ in the latter, but our focus is on 
tuning the standard deviation parameter to determine a statistic describing how a given data point is placed within its class. The $\gamma$-persistence then gives a measurement similar to distance to the boundary but that is drawn from sampling instead of distance. This sampling allows for a better description of the local geometry of the class and decision boundary, as we will see in Section \ref{subsec:stab}. Our statistic is based on the fraction of a Gaussian sampling of the neighborhood of a point that receives the same classification; this is different from that of ~\citet{roth19aodds}, which is the expected error of the original data points logit and the perturbed data points logit. Additionally, while their statistics are defined pairwise with reference to pre-chosen original and candidate classes, ours is not.

%%%%%%%%%%%%%%%%%%%%%%%%%%%%%%%%%%%%%%%%%%%%%%%%%%%%%%%%%%%%%%%%
% persistence paper
%%%%%%%%%%%%%%%%%%%%%%%%%%%%%%%%%%%%%%%%%%%%%%%%%%%%%%%%%%%%%%%%

% This article proposes a new framework for studying adversarial examples that does not depend directly on the distance to the decision boundary. 
% Similarly to the smoothed classifier literature, we define a (natural or adversarial) data point to be $(\gamma,\sigma)$-stable if the probability of the same classification is at least $\gamma$ for points sampled in a Gaussian neighborhood of the point with a given standard deviation $\sigma$. 
% We focus on studying the differences between persistence metrics along interpolants of natural and adversarial points.
% We show that adversarial examples have significantly lower persistence than natural examples for large neural networks in the context of the MNIST and ImageNet datasets. 
% We connect this lack of persistence with decision boundary geometry by measuring angles of interpolants with respect to decision boundaries.
% Finally, we connect this approach with robustness by developing a manifold alignment gradient metric and demonstrating the increase in robustness that can be achieved when training with the addition of this metric. 

\section{Introduction}

Deep Neural Networks (DNNs) and their variants are core to the success of modern machine learning ~\citep{prakash2018}, and have dominated competitions in image processing, optical character recognition, object detection, video classification, natural language processing, and many other fields \citep{SCHMIDHUBER201585}. Yet such classifiers are notoriously susceptible to manipulation via adversarial examples ~\citep{szegedy2013}. Adversarial examples occur when natural data can be subject to subtle perturbation which results in substantial changes in output. Adversarial examples are not just a peculiarity, but seem to occur for most, if not all, DNN classifiers. For example, \citet{shafahi2018are} used isoperimetric inequalities on high dimensional spheres and hypercubes to conclude that there is a reasonably high probability that a correctly classified data point has a nearby adversarial example. This has been reiterated using mixed integer linear programs to rigorously check minimum distances necessary to achieve adversarial conditions ~\citep{tjeng2017evaluating}. In addition, \citet{ilyas2019adversarial} showed that adversarial examples can arise from features that are good for classification but not robust to perturbation. 

There have been many attempts to identify adversarial examples using
properties of the decision boundary. \citet{Fawzi2018empirical} found
that decision boundaries tend to have highly curved regions, and these
regions tend to favor negative curvature, indicating that regions that
define classes are highly non-convex. The purpose of this work is to
investigate these geometric properties related to the decision
boundaries. We will do this by proposing a notion of stability that is
more nuanced than simply measuring distance to the decision boundary,
and is also capable of elucidating information about the curvature of
the nearby decision boundary. We develop a statistic extending
prior work on smoothed classifiers by \citet{cohen2019certified}. We denote this metric as Persistence and use it as a measure of how far away from a point one can go via Gaussian sampling and still consistently find points with the same classification. One advantage of this statistic is that it is easily estimated by sending a Monte Carlo sampling about the point through the classifier. In combination with this metric, direct measurement of decision boundary incidence angle with dataset interpolation and manifold alignment can begin to complete the picture for how decision boundary properties are related with neural network robustness. 

 These geometric properties are related to the alignment of gradients
 with human perception \citep{ganz2022perceptually,
   kaur2019perceptually, shah2021input} and with the related
 underlying manifold \citep{kaur2019perceptually,
   ilyas2019adversarial} which may imply robustness. For our purposes,
 Manifold Aligned Gradients (MAG) will refer to the property that the
 gradients of a model with respect to model inputs follow a given data
 manifold $\mathcal{M}$ extending similar relationships from other
 work by \citet{shamir2021dimpled}. 

{\bf Contributions.} We believe these geometric properties are related
to why smoothing methods have been useful in robustness tasks
~\citep{cohen2019certified, lecuyer2019certified, li2019certified}. We propose three approaches in order to connect robustness with geometric properties of the decision boundary learned by DNNs: 
\begin{enumerate}
    \item We propose and implement two metrics based on the success of smoothed classification techniques:  $(\gamma,\sigma)$-stability and $\gamma$-persistence defined with reference to a classifier and a given point (which can be either a natural or adversarial image, for example) and demonstrate their validity for analyzing adversarial examples. 
    \item We interpolate across decision boundaries using our persistence metric to demonstrate an inconsistency at the crossing of a decision boundary when interpolating from natural to adversarial examples.
    \item We demonstrate via direct interpolation across decision boundaries and measurement of angles of interpolating vectors relative to the decision boundary itself that dimensionality is not solely responsible for geometric vulnerability of neural networks to adversarial attack. 
%    \item We show that robustness tends to correspond with MAG and
%      that directly optimizing a MAG metric improves robustness
%      against linear attacks. We also show that this latter
%      implication does not hold for non-linear adversaries.  
\end{enumerate}

\section{Motivation and related work}

Our work is intended to shed light on the existence and prevalence of adversarial examples to DNN classifiers. It is closely related to other attempts to characterize robustness to adversarial perturbations, and here we give a detailed comparison.

{\bf Distance-based robustness.}

A typical approach to robustness of a classifier is to consider
distances from the data manifold to the decision boundary
~\citep{Wang2020Improving, xu2023exploring, he2018decision}.
\citet{khoury2018} define a classifier to be robust if the class of
each point in the data manifold is contained in a sufficiently large
ball that is entirely contained in the same class. The larger the
balls, the more robust the classifier. It is then shown that if
training sets are sufficiently dense in relation to the reach of the
decision axis, the classifier will be robust in the sense that it
classifies nearby points correctly. In practice, we do not know that
the data is so well-positioned, and it is quite possible, especially
in high dimensions, that the reach is extremely small, as evidenced by
results on the prevalence of adversarial examples, e.g.,
\citet{shafahi2018are} and in evaluation of ReLU networks with mixed
integer linear programming e.g., ~\citet{tjeng2017evaluating}.

\citet{tsipras2018robustness} investigated robustness in terms of how
small perturbations affect the the average loss of a classifier. They
define standard accuracy of a classifier in terms of how often it
classifies correctly, and robust accuracy in terms of how often an
adversarially perturbed example classifies correctly. It was shown
that sometimes accuracy of a classifier can result in poor robust
accuracy. \citet{gilmer2018adversarial} use the expected distance to
the nearest different class, when drawing a data point from the data
distribution, to capture robustness. They then show that an accurate
classifier can result in a small distance to the nearest different
class in high dimensions when the data is drawn from concentric
spheres. Many recent works ~\citep{he2018decision, chen2023aware,
  jin2022roby} have linked robustness with decision boundary dynamics,
both by augmenting training with data near decision boundaries, or
with dynamics related to distances from decision boundaries. We
acknowledge the validity of this work, but will address some of its
primary limitations by carefully studying the orientation
of the decision boundary relative to model data.

A related idea is that adversarial examples often arise within cones,
outside of which images are classified in the original class, as
observed by \citet{roth19aodds}. Many theoretical models of
adversarial examples, for instance the dimple model developed by
\citet{shamir2021}, have high curvature and/or sharp corners as an
essential piece of why adversarial examples can exists very close to
natural examples. 

{\bf Adversarial detection via sampling.}
While adversarial examples often occur, they still may be rare in the
sense that most perturbations do not produce adversarial
examples. \citet{yu2019new} used the observation that adversarial
examples are both rare and close to the decision boundary to detect
adversarial examples. They take a potential data point and look to see
if nearby data points are classified differently than the original
data point after only a few iterations of a gradient descent
algorithm. If this is true, the data point is likely natural and if
not, it is likely adversarial. This method has been generalized with
the developing of smoothed classification methods
~\citep{cohen2019certified, lecuyer2019certified, li2019certified}
which at varying stages of evaluation add noise to the effect of
smoothing output and identifying adversaries due to their higher
sensitivity to perturbation.. These methods suffer from significant
computational complexity ~\citep{kumar2020curse} and have been shown
to have fundamental limitations in their ability to rigorously certify
robustness ~\citep{blum2020random, yang2020randomized}. We will generalize this approach into a metric which will allow us to directly study these limitations in order to better understand how geometric properties have given rise to adversarial vulnerabilities. In general, the results of \citet{yu2019new} indicate that considering samples of nearby points, which approximate the computation of integrals, is likely to be more successful than methods that consider only distance to the decision boundary.

\citet{roth19aodds} proposed a statistical method to identify adversarial examples from natural data. Their main idea was to consider how the last layer in the neural network (the logit layer) would behave on small perturbations of a natural example. %, i.e., on $x+\varepsilon n$ where $x$ is a natural example, $\varepsilon>0$ is small, and $n \sim N(0,I)$.  
This is then compared to the behavior of a potential adversarial example. 

It was shown by \citet{hosseini2019odds} that it is possible to produce adversarial examples, for instance using a logit mimicry attack, that instead of perturbing an adversarial example toward the true class, actually perturb to some other background class. In fact, we will see in Section \ref{sec:mnist} that the emergence of a background class, which was observed as well by \citet{roth19aodds}, is quite common. Although many recent approaches have taken advantage of these facts ~\citep{taori2020shifts, lu2022randommasking, Osada_2023_WACV, blau2023classifier} in order to measure and increase robustness, we will leverage these sampling properties to develop a metric directly on decision-boundary dynamics and how they relate to the success of smoothing based robustness. 

{\bf Manifold Aware Robustness}

The sensitivity of convolutional neural networks to imperceptible changes in input has thrown into question the true generalization of these models.
\citet{jo2017measuring} study the generalization performance of CNNs by transforming natural image statistics.% surface level irregularities
Similarly to our MAG approach, they create a new dataset with well-known properties to allow the testing of their hypothesis.
They show that CNNs focus on high level image statistics rather than human perceptible features.
This problem is made worse by the fact that many saliency methods fail basic sanity checks \citep{adebayo2018sanity, kindermans2019reliability}.

Until recently, it was unclear whether robustness and manifold alignment were directly linked, as the only method to achieve manifold alignment was adversarial training.
Along with the discovery that smoothed classifiers are perceptually
aligned, comes the hypothesis that robust models in general share this
property put forward by \citet{kaur2019perceptually}.
This discovery raises the question of whether this relationship
between manifold alignment of model gradients and robustness is bidirectional.

\citet{khoury2018} study the geometry of natural images, and create a lower bound for the number of data points required to effectively capture variation in the manifold.
They demonstrate that this lower bound is so large as to be intractable for all practical datasets, severely limiting robustness guarantees that rely on sampling from the data distribution.
\citet{vardi2022gradient} demonstrate that even models that satisfy strong conditions related to max margin classifiers are implicitly non-robust. 
Instead of relying on sampling, \citet{shamir2021dimpled} propose using the tangent space of a generative model as an estimation of this manifold. 
This approach reduces the sampling requirements; however, this approach relies on a generative model which will, by definition, suffer from the same dimensionality limitations.
On the other hand, \citet{magai2022topology} thoroughly review certain topological properties to demonstrate that neural networks intrinsically use relatively few dimensions of variation during training and evaluation. 
It is possible that the this ability to reduce degrees of freedom is the mechanism which resolves the curse of dimensionality inherent to many deep learning problems, especially when domain specific constraints are used to define the data manifold.
PCA and manifold metrics have been recently used to identify adversarial examples \citep{aparne2022pca, nguyen-minh-luu-2022-textual}. 
We will extend this work to study the relationship between robustness and manifold alignment directly by baking alignment directly into networks and comparing them with other forms of robustness. 

{\bf Summary.}
 In Sections \ref{sec:meth} and \ref{sec:experiments}, we will investigate stability of both natural data and adversarial examples by considering sampling from Gaussian distributions centered at a data point with varying standard deviations. Using the standard deviation as a parameter, we are able to derive a statistic for each point that captures how entrenched it is in its class in a way that is less restrictive than the robustness described by \citet{khoury2018}, takes into account the rareness of adversarial examples described by \citet{yu2019new}, builds on the idea of sampling described by \citet{roth19aodds} and \citet{hosseini2019odds}, and represent curvatures in a sense related to \citet{Fawzi2018empirical}. Furthermore, we will relate these stability studies to direct measurement of interpolation incident angles with decision boundaries in Subsection~\ref{subsec:db} and ~\ref{subsec:dbe} and the effect of reduction of data onto a known lower dimensional manifold in Subsections ~\ref{subsec:ma} and ~\ref{subsec:mae}.  

\section{Methods} \label{sec:meth} % Stability and Persistence

In this section we will lay out the theoretical framework for studying stability, persistence, and decision boundary crossing-angles. 

\subsection{Stability and Persistence} \label{subsec:stab}
In this section we define a notion of stability of classification of a point under a given classification model. In the following, $X$ represents the ambient space the data is drawn from (typically $\RR^n$) even if the data lives on a sub-manifold of $X$, and $L$ is a set of labels (often $\{1,\dots,\ell\}$).  Note that points $x\in X$ can be natural or adversarial points.%The following definition complements the definition for adversarial examples by providing a criteria for the local stability of the classifier about a point, which could be an actual test point or an adversarial example: 

\begin{definition}
Let $\CC:X\to L$ be a classifier, $x \in X$, $\gamma\in(0,1)$, and $\sigma>0$. We say $x$ is \emph{$(\gamma,\sigma)$-stable} with respect to $\CC$ if $\mathbb{P}[\CC(x')=\CC(x)] \geq \gamma$ for $x' \sim \rho = N(x, \sigma^2 I)$; i.e. $x'$ is drawn from a Gaussian with variance $\sigma^2$ and mean $x$.
\end{definition}

In the common setting when $X=\RR^n$, we have
\[\mathbb{P}[\CC(x')=\CC(x)] = \int_{\RR^n} \mathbbm{1}_{\CC^{-1}(\CC(x))} (x') d\rho (x') = \rho(\CC^{-1}\CC(x)).\]
Note here that $\CC^{-1}$ denotes pre-image. %In the case of images drawn from $\RR^n$, we can write this integral precisely as
One could substitute various probability measures $\rho$ above with mean $x$ and variance $\sigma^2$ to obtain different measures of stability corresponding to different ways of sampling the neighborhood of a point.  Another natural choice would be sampling the uniform measure on balls of changing radius. Based on the concentration of measure for both of these families of measures we do not anticipate significant qualitative differences in these two approaches. We propose Gaussian sampling because it is also a product measure, which makes it easier to sample and simplifies some of our calculations. Figure ~\ref{fig:pain} below compares the uniform measure on balls with our $\gamma-$softened approach and its insensitivity to rare events which in this case highlights more macro-scale features which would be missed by a hard ball. 

For the Gaussian measure, the probability above may be written more concretely as
\begin{equation}\label{EQN:Gaussian}
\frac{1}{\left(\sqrt{2\pi}\sigma\right)^{n}} \int_{\RR^n} \mathbbm{1}_{\CC^{-1}(\CC(x))} (x')e^{-\frac{\norm{x - x'}^2}{2\sigma^2}} dx'.
\end{equation}
In this work, we will conduct experiments in which we estimate this stability for fixed $(\gamma,\sigma)$ pairs via a Monte Carlo sampling, in which case the integral \eqref{EQN:Gaussian} is approximated by taking $N$ i.i.d. samples $x_k \sim \rho$ and computing
\[
    \frac{\norm{x_k : \CC(x_k) = \CC(x)}}{N}.
\]
Note that this quantity converges to the integral \eqref{EQN:Gaussian} as $N\to\infty$ by the Law of Large Numbers.

The ability to adjust the quantity $\gamma$ is important because it is much weaker than a notion of stability that requires a ball that stays away from the decision boundary as by \citet{khoury2018}. By choosing $\gamma$ closer to $1$, we can require the samples to be more within the same class, and by adjusting $\gamma$ to be smaller we can allow more overlap.

We also propose a related statistic, \emph{persistence}, by fixing a particular $\gamma$ and adjusting $\sigma$. For any $x\in X$ not on the decision boundary, for any choice of $0<\gamma<1$ there exists a $\sigma_\gamma$ small enough such that if $\sigma < \sigma_\gamma$ then $x$ is $(\gamma,\sigma)$-stable. We can now take the largest such $\sigma_\gamma$ to define persistence.

\begin{definition}
    Let $\CC:X\to L$ be a classifier, $x \in X$, and $\gamma\in(0,1)$. Let $\sigma_\gamma^*$ be the maximum $\sigma_\gamma$ such that $x$ is $(\gamma, \sigma)$-stable with respect to $\CC$ for all $\sigma<\sigma_\gamma$. We say that $x$ has \emph{$\gamma$-persistence} $\sigma_\gamma^*$.
\end{definition}

The $\gamma$-persistence quantity $\sigma_\gamma^*$ measures the stability of the neighborhood of a given $x$ with respect to the output classification. Small persistence indicates that the classifier is unstable in a small neighborhood of $x$, whereas large persistence indicates stability of the classifier in a small neighborhood of $x$. In the later experiments, we have generally taken $\gamma = 0.7$. This choice is arbitrary and chosen to fit the problems considered here. In our experiments, we did not see significant change in results with small changes in the choice of $\gamma$, see Appendix Figure~\ref{fig:persistencediffgamma}.

In our experiments, we numerically estimate $\gamma$-persistence via a bisection algorithm that we term the Bracketing Algorithm. Briefly, the algorithm first chooses search space bounds $\sigma_{\min}$ and $\sigma_{\max}$ such that $x$ is  $(\gamma,\sigma_{\min})$-stable but is not $(\gamma,\sigma_{\max})$-stable with respect to $\CC$, and then proceeds to evaluate stability by bisection until an approximation of $\sigma_\gamma^*$ is obtained.

Another natural approach would be to compute and sample a distribution tailored for the geometry around our chosen point instead of a gaussian. This could be done using non-parametric estimation or by performing an iterative computation of eigenvalues and eigenvectors. In this case, we could use a selection of eigenvalues to provide different $\sigma_{\gamma, i}$ for each eigenvector corresponding to the modes of variation in the distribution. It may be possible to use a tangent kernel approximation of a neural network like that proposed in ~\cite{bell2023} to compute these modes of variation directly from model parameters. We believe that this approach will likely perform well in further mapping geometric properties of neural networks, although we will show that our persistence metric performs well despite lacking this more structured information. 

\begin{figure}[!ht]
\begin{center}
    \includegraphics[width=0.9\textwidth]{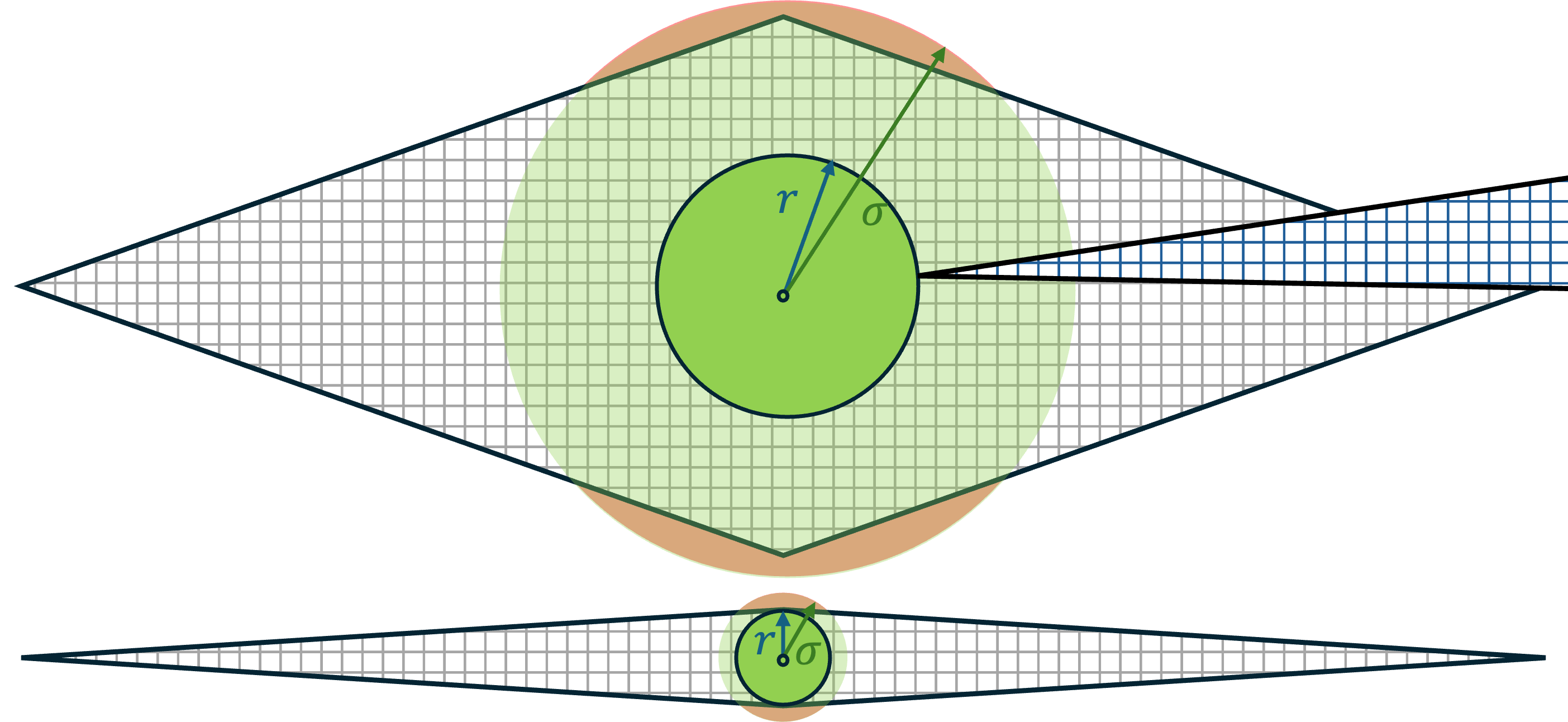}
\end{center}
\caption{Both hard (middle) and soft (transparent outer) balls conform to the most constricted geometric region (bottom). However, the hard ball is also sensitive to rare counter-examples (above spike). By working with a soft ball, we need only compute a likelihood rather than an exhaustive search for rare examples. In addition, the soft-ball still responds correctly to macro-scale geometric features.}
\label{fig:pain}
\end{figure}

\subsection{Decision Boundaries} \label{subsec:db}

In order to examine decision boundaries and their properties, we will carefully define the decision boundary in a variety of equivalent formulations. 

\subsubsection{The Argmax Function Arises in Multi-Class Problems}

A central issue when writing classifiers is mapping from continuous outputs or probabilities as shown in ~\ref{fig:pdb} to discrete sets of classes. Frequently argmax type functions are used to accomplish this mapping. To discuss decision boundaries, we must precisely define argmax and some of its properties. 

In practice, argmax is not strictly a function, but rather a mapping from the set of outputs or activations from another model into the power set of a discrete set of classes:

\begin{equation}
    \text{argmax} : \R^k \to \mathcal{P}(L)
\end{equation}

Defined this way, we cannot necessarily consider $\text{argmax}$ to be a function in general as the singleton outputs of argmax overlap in an undefined way with other sets from the power set. However, if we restrict our domain carefully, we can identify certain properties. 
\begin{figure}[!ht]
\begin{center}
\begin{tikzpicture}
  \draw[->] (-0.5, 0) -- (3.5, 0) node[right] {$x$};
  \draw[->] (0, -0.5) -- (0, 3.5) node[above] {$y$};
  \node at (1, 2) (c1) {Class 1};
  \node at (2, 1) (c2) {Class 2};
  \draw[-, fill, blue, opacity=.3] (0, 3) -- (3, 3) -- (0, 0);
  \draw[-, fill, red, opacity=.3] (3, 0) -- (3, 3) -- (0, 0);
  \draw[scale=1, domain=0:3, smooth, variable=\x, black, line width=0.45mm] plot ({\x}, {\x});
  %\draw[scale=0.5, domain=-3:3, smooth, variable=\y, red]  plot ({\y*\y}, {\y});
\end{tikzpicture}\begin{tikzpicture}
  \draw[->] (-0.5, 0) -- (3.5, 0) node[right] {$x$};
  \draw[->] (0, -0.5) -- (0, 3.5) node[above] {$y$};
  \node at (1, 2) (c1) {Class 1};
  \node at (2, 1) (c2) {Class 2};
  \draw[-, fill, blue, opacity=.3] (0, 3) -- (3, 3) -- (0, 0);
  \draw[-, fill, red, opacity=.3] (3, 0) -- (3, 3) -- (0, 0);
  \draw[scale=1, domain=0:3, smooth, variable=\x, black, line width=0.45mm] plot ({\x}, {\x});
  \draw[-, orange, line width=0.45mm] (0,3) -- (3, 0);
  %\draw[scale=0.5, domain=-3:3, smooth, variable=\y, red]  plot ({\y*\y}, {\y});
\end{tikzpicture}

\caption{Decision boundary in $[0,1] \times [0,1]$ (left) and decision boundary restricted to probabilities (right). If the output of $F$ are \emph{probabilities} which add to one, then all points of $x$ will map to the orange line (right). We note that the point $(0.5, 0.5)$ is therefore the only point on the decision boundary for probability valued $F$. We may generalize to higher dimensions where all probability valued models $F$ will map into the the plane $x + y + z + \cdots = 1$ in $Y$ and the decision boundary will be partitioned into $K-1$ components, where the $K$-decision boundary is the intersection of this plane with the \emph{centroid} line $x = y = z = \cdots$ and the $2$-decision boundaries become planes intersecting at the same line. }
\label{fig:pdb}
\end{center}
\end{figure}
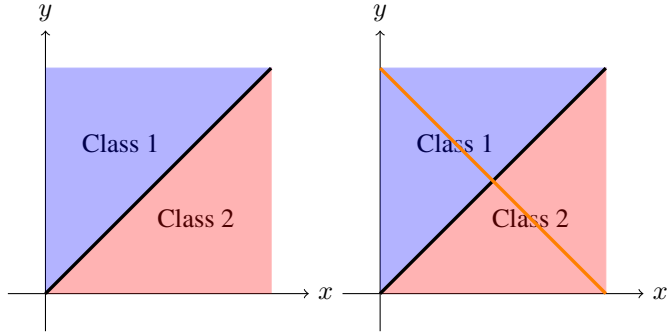
Restricting to only the pre-image of the singletons, it should be
clear that argmax is constant.  Indeed, restricted to the pre-image of
any set in the power-set, argmax is constant and thus
continuous. This induces the discrete topology whereby the pre-image
of an individual singleton is open. Observe that for any point whose
image is a singleton, one element of the domain vector must exceed the 
others by $\varepsilon > 0$. We shall use the $\ell^1$ metric for
distance, and thus if we restrict ourselves to a ball of radius
$\varepsilon$, then all elements inside this ball will have that
element still larger than the rest and thus map to the same singleton
under argmax. Since the union of infinitely many open sets is open in
$\R^k$, the union of all singleton pre-images is an open
set. Conveniently this also provides proof that the union of all of
the non-singleton sets in $\mathcal{P}(C)$ is a closed set. We will
call this closed set the argmax Decision Boundary. We will list two
equivalent formulations for this boundary.  

\paragraph{Complement Definition}

A point $x$ is in the \emph{decision interior} $D_{\mathcal{C}}'$ for
a classifier $\mathcal{C}: \mathbb{R}^N \rightarrow L$ if there exists $\delta
> 0$ such that $\forall \epsilon < \delta$, the number of elements $n(\mathcal{C}(B_\epsilon(x))) = 1$. 

The \emph{decision boundary} of a classifier $\mathcal{C}$ is the closure of the complement of the decision interior $\overline{\{x : x \notin D_{\mathcal{C}}'\}}$. 

\paragraph{Level Set Definition}

For an input space $X$, the decision boundary $D \subset X$ of a probability valued function $f$ is the pre-image of a union of all level sets of outputs $f(X) = {c_1, c_2, ..., c_k}$ defined by a constant $c$ such that for some set of indices $I$, we have $c = c_i$ for every $i$ in $I$ and $c > c_j$ for every $j$ not in $I$. The pre-image of each such set are all $x$ such that $f(x) = A_c$ for some $c$. 

\section{Experiments} \label{sec:experiments}

In this section we investigate the stability and persistence behavior of natural and adversarial examples for MNIST \citep{MNIST} and ImageNet \citep{ILSVRC15} using a variety of different classifiers. For each set of image samples generated for a particular dataset, model, and attack protocol, we study $(\gamma,\sigma)$-stability and $\gamma$-persistence of both natural and adversarial images, and also compute persistence along trajectories from natural to adversarial images. In general, we use $\gamma = 0.7$ as a balance of being large enough that it is always possible to pick a $\sigma$ that yields a set of samples whose fraction of the original class is meaningfully less than $\gamma$, and small enough that $\sigma$ is likely to converge to a relatively precise number. While observed behavior near decision boundaries does not change much for small changes in $\gamma$ (e.g. 0.7 vs 0.8 vs 0.9, see Appendix Figure~\ref{fig:persistencediffgamma}), we note that as $\gamma$ approaches 1, the number of samples needed to compute $\sigma$ with precision rises asymptotically since the likelihood of finding sufficient samples not in the original class goes to zero. While most of the adversarial attacks considered here have a clear target class, the measurement of persistence does not require considering a particular candidate class.  Furthermore, we will evaluate decision boundary incidence angles and apply our conclusions to evaluate models trained with manifold aligned gradients. 

\subsection{MNIST Experiments}

Since MNIST is relatively small compared to ImageNet, we trained several classifiers with various architectures and complexities and implemented the adversarial attacks directly. Adversarial examples were generated against each of these models using Iterative Gradient Sign Method (IGSM \citep{kurakin_adversarial_2016}) and Limited-memory Broyden-Fletcher-Goldfarb-Shanno (L-BFGS \citep{liu1989limited}).

\subsubsection{Investigation of $(\gamma, \sigma)$-stability on MNIST}\label{sec:mnist}

We begin with a fully connected ReLU network with layers of size 784, 100, 20, and 10 and small regularization $\lambda = 10^{-7}$ which is trained on the standard MNIST ~\cite{MNIST} training set. We then start with a randomly selected MNIST test image $x_1$ from the \texttt{1}'s class and generate adversarial examples $x_0,x_2,\dots,x_9$ using IGSM for each target class other than \texttt{1}. The neighborhoods around each $x_i$ are examined by generating 1000 i.i.d. samples from $N(x_i,\sigma^2I)$ for each of 100 equally spaced standard deviations $\sigma\in(0,1.6)$. Figure \ref{fgsmo} shows the results of the Gaussian perturbations of a natural example $x_1$ of the class labeled \texttt{1} and the results of Gaussian perturbations of the adversarial example $x_0$ targeted at the class labeled \texttt{0}. We provide other examples of $x_2,\ldots,x_9$ in the supplementary materials. Note that the original image is very stable under perturbation, while the adversarial image is not. 

\begin{figure}[!ht]
  \includegraphics[width = .49\textwidth]{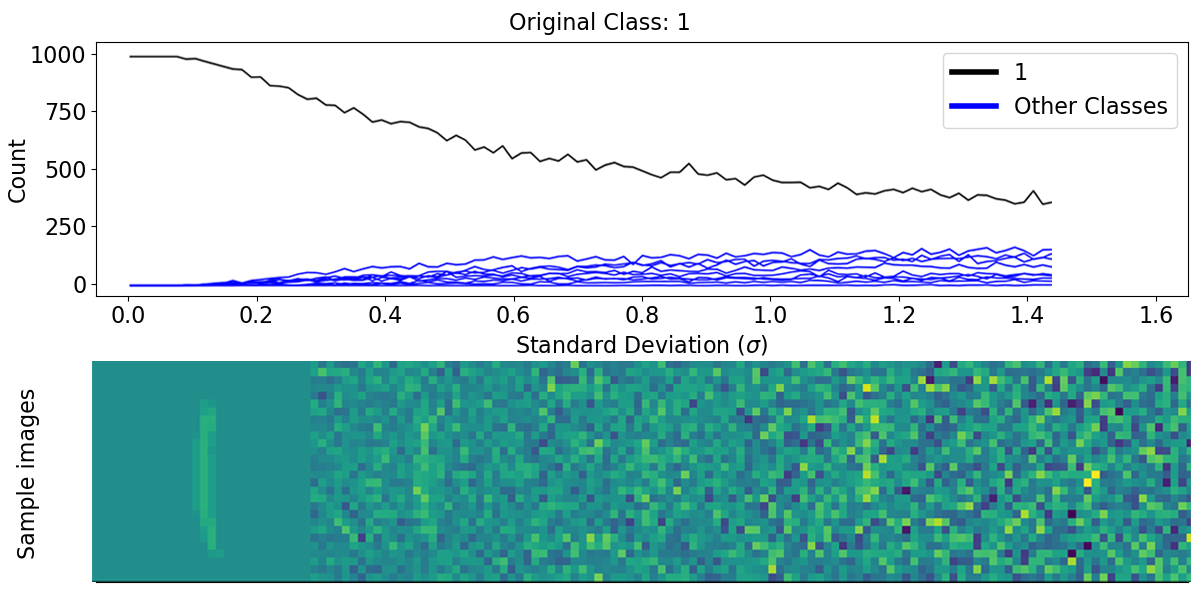}
  \includegraphics[width = .49\textwidth]{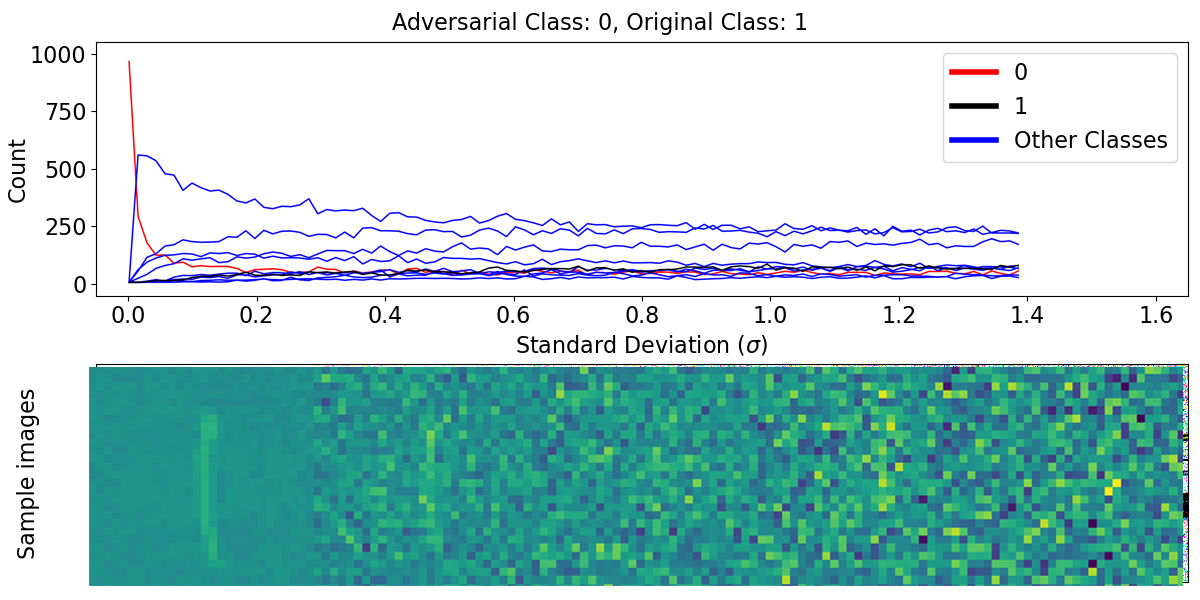}
\caption{Frequency of each class in Gaussian samples with increasing variance around a natural image of class \texttt{1} (left) and around an adversarial attack of that image targeted at \texttt{0} generated using IGSM (right). The adversarial class (\texttt{0}) is shown as a red curve. The natural image class (\texttt{1}) is shown in black. Bottoms show example sample images at different standard deviations for natural (left) and adversarial (right) examples.}\label{fgsmo}
\end{figure}

\subsubsection{Persistence of adversarial examples for MNIST}

To study persistence of adversarial examples on MNIST, we take the same network architecture as in the previous subsection and randomly select 200 MNIST images. For each image, we used IGSM to generate 9 adversarial examples (one for each target class) yielding a total of 1800 adversarial examples. In addition, we randomly sampled 1800 natural MNIST images. For each of the 3600 images, we computed $0.7$-persistence; the results are shown in Figure \ref{fig:IGSMpersistenceMNIST}. One sees that $0.7$-persistence of adversarial examples tends to be significantly smaller than that of natural examples for this classifier, indicating that they are generally less stable than natural images. We will see subsequently that this behavior is typical.

\begin{figure}[!ht]
\centering
\includegraphics[trim=10 8 10 10, clip,width=.8\textwidth]{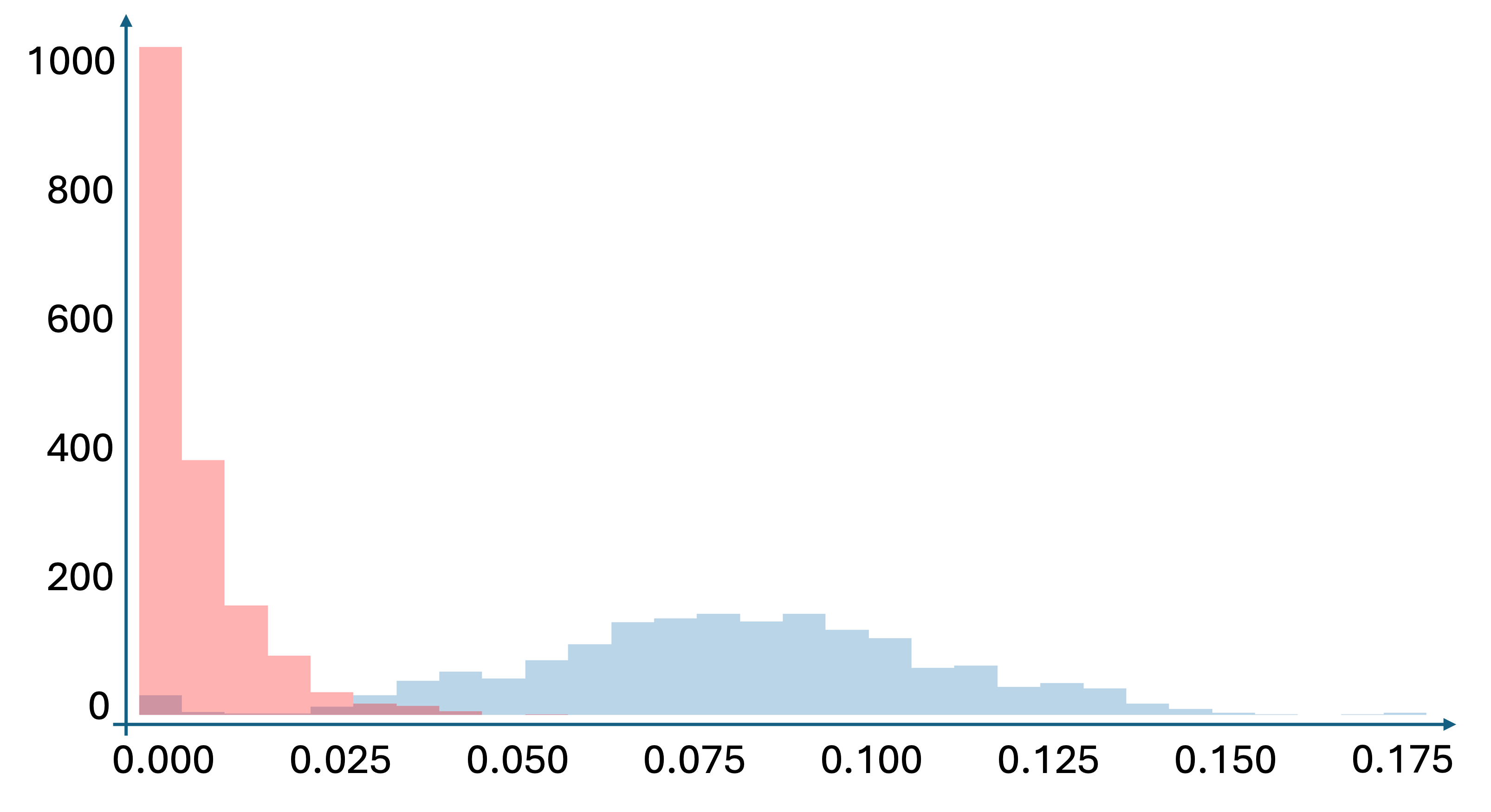}
\caption{Histogram of $0.7$-persistence of IGSM-based adversarial examples (red) and natural examples (blue) on MNIST. %The histogram shows that $0.7$-persistence for adversarial examples tends to be smaller than $0.7$-persistence for natural examples.
}

\label{fig:IGSMpersistenceMNIST}
\end{figure}

Next, we investigate the relationship of network complexity and $(\gamma,\sigma)$-stability by revisiting the now classic work of \citet{szegedy2013} on adversarial examples. 

Table \ref{table1} recreates and adds on to part of \cite[Table 1]{szegedy2013} in which networks of differing complexity are trained and attacked using L-BFGS. The table contains new columns showing the average $0.7$-persistence for both natural and adversarial examples for each network, as well as the average distortion for the adversarial examples. The distortion is the $\ell^2$-norm divided by square root of the dimension $n$. The first networks listed are of the form FC10-k, and are fully connected single layer ReLU networks that map each input vector $x \in \RR^{784}$ to an output vector $y \in \RR^{10}$ with a regularization added to the objective function of the form $\lambda\Norm{w}_2/N$, where $\lambda = 10^{-k}$ and $N$ is the number of parameters in the weight vector $w$ defining the network. The higher values of $\lambda$ indicate more regularization.  

FC100-100-10 and FC200-200-10 are networks with 2 hidden layers (with 100 and 200 nodes, respectively) with regularization added for each layer and $\lambda$ for each layer equal to $10^{-5}, 10^{-5}$, and  $10^{-6}$. Training for these networks was conducted with a fixed number of epochs (typically 21). For the bottom half of Table \ref{table1}, we also considered networks with four convolutional layers plus a max-pooling layer connected by ReLU to a fully connected hidden layer with increasing numbers of channels denoted as as ``C-Ch,'' where C reflects that this is a CNN and Ch denotes the number of channels. A more detailed description of these networks can be found in Appendix \ref{appendix:CNNs}.

\begin{table}[ht]
\centering
\caption{Recreation of \citet[Table 1]{szegedy2013} for the MNIST dataset.  For each network, we show Testing Accuracy (in \%), Average Distortion ($\|x\|_2/\sqrt{n}$) of adversarial examples, and new columns show average $0.7$-persistence values for natural (Nat) and adversarial (Adv) images. 300 natural and 300 adversarial examples generated with L-BFGS were used for each aggregation.}
\label{table1}
\begin{tabular}{lllll}
\toprule
Network & Test Acc & Avg Dist & Persist (Nat) & Persist (Adv) \\
\midrule
FC10-4 & 92.09 & 0.123 & 0.93 & 1.68\\
FC10-2 & 90.77 & 0.178 & 1.37 & 4.25\\
FC10-0 & 86.89 & 0.278 & 1.92 & 12.22\\
FC100-100-10 & 97.31 & 0.086 & 0.65 & 0.56 \\
FC200-200-10 & 97.61 & 0.087 & 0.73 & 0.56 \\
\midrule
C-2 & 95.94 & 0.09 & 3.33 & 0.027 \\
C-4 & 97.36 & 0.12 & 0.35 & 0.027 \\
C-8 & 98.50 & 0.11 & 0.43  & 0.0517 \\
C-16 & 98.90 & 0.11 & 0.53 & 0.0994 \\
C-32 & 98.96 & 0.11 & 0.78 & 0.0836 \\
C-64 & 99.00 & 0.10 & 0.81 & 0.0865 \\
C-128 & 99.17 & 0.11 & 0.77 & 0.0883 \\
C-256 & 99.09 & 0.11  & 0.83 & 0.0900 \\
C-512 & 99.22 & 0.10 & 0.793 & 0.0929 \\

\bottomrule
\end{tabular}
\end{table}

The main observation from Table \ref{table1} is that for higher complexity networks,
adversarial examples tend to have smaller persistence than natural examples. Histograms reflecting these observations can be found in the supplemental material. %This can be seen as well in Figure \ref{fig:FC200-200-10}, which shows the $0.7$-persistence for natural and adversarial examples for the network FC200-200-10. 
Another notable takeaway is that for models with fewer effective parameters, the attack distortion necessary to generate a successful attack is so great that the resulting image is often more stable than a natural image under that model, as seen particularly in the FC10 networks. Once there are sufficiently many parameters available in the neural network, we found that both the average distortion of the adversarial examples and the average $0.7$-persistence of the adversarial examples tended to be smaller. This observation is consistent with the idea that networks with more parameters are more likely to exhibit decision boundaries with more curvature. 

\subsection{Results on ImageNet}

For ImageNet \citep{Imagenet-old}, we used pre-trained ImageNet classification models, including alexnet \citep{alexnet} and vgg16 \citep{simonyan2014very}.

We then generated attacks based on the ILSVRC 2015 \citep{ILSVRC15}
validation images for each of these networks using a variety of modern
attack protocols, including Fast Gradient Sign Method (FGSM
\citep{goodfellow_explaining_2014}), Momentum Iterative FGSM (MIFGSM
\citep{dongMIFGSM}), Basic Iterative Method (BIM
\citep{kurakin_adversarial_2016}), Projected Gradient Descent (PGD
\citep{madry_towards_2017}), Randomized FGSM (R+FGSM
\citep{tramer2018ensemble}), and Carlini-Wagner (CW
~\citet{carlini_towards_2016}). These were all generated using the
TorchAttacks by \citet{kim2021torchattacks} tool-set.

\subsubsection{Investigation of $(\gamma, \sigma)$-stability on ImageNet}

In this section, we show the results of Gaussian neighborhood sampling in ImageNet. Figures \ref{fig:imagenet_adv} and \ref{fig:persistent_interpimage} arise from vgg16 and adversarial examples created with BIM; results for other networks and attack strategies are similar, with additional figures in the supplementary material. Figure \ref{fig:imagenet_adv} (left) begins with an image $x$ with label \texttt{goldfinch}. For each equally spaced $\sigma\in(0,2)$, 100 i.i.d. samples were drawn from the Gaussian distribution $N(x,\sigma^2I)$, and the counts of the vgg16 classification for each label are shown. In Figure \ref{fig:imagenet_adv} (right), we see the same plot, but for an adversarial example targeted at the class \texttt{indigo\_bunting}, which is another type of bird, using the BIM attack protocol. %There are similar results with other attack protocols, as described in the supplementary materials.

The key observation in Figure \ref{fig:imagenet_adv} is that the frequency of the class of the adversarial example (\texttt{indigo\_bunting}, shown in red) falls off much quicker than the class for the natural example (\texttt{goldfinch}, shown in black). In this particular example, the original class appears again after the adversarial class becomes less prevalent, but only for a short period of $\sigma$, after which other classes begin to dominate. In some examples the original class does not dominate at all after the decline of the adversarial class. The adversarial class almost never dominates for a long period of $\sigma$. 

\begin{figure}[ht]
\centering
\includegraphics[width = .49\textwidth]{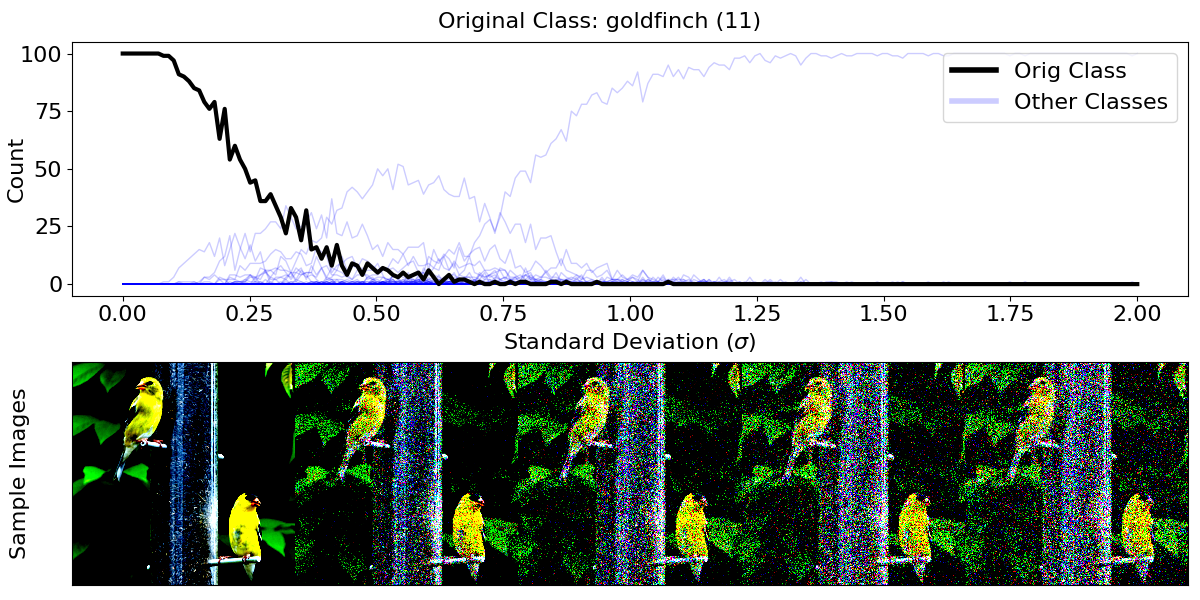}
\includegraphics[width = .49\textwidth]{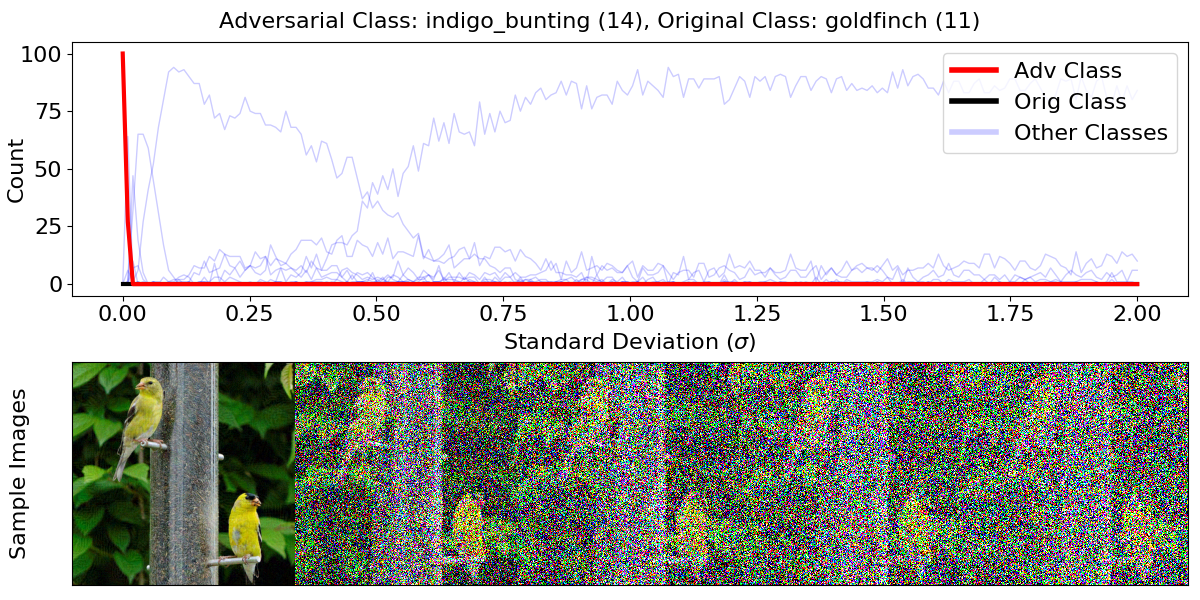}

\caption{Frequency of each class in Gaussian samples with increasing variance around a \texttt{goldfinch} image (left) and an adversarial example of that image targeted at the \texttt{indigo\_bunting} class and calculated using the BIM attack (right). Bottoms show example sample images at different standard deviations for natural (left) and adversarial (right) examples.}
\label{fig:imagenet_adv}
\end{figure}

\subsubsection{Persistence of adversarial examples on ImageNet}

Figure \ref{fig:persistent_interpimage} shows a plot of the $0.7$-persistence along the straight-line path between a natural example and adversarial example as parameterized between $0$ and $1$. It can be seen that the drop off of persistence occurs precisely around the decision boundary. This indicates some sort of curvature favoring the class of the natural example, since otherwise the persistence would be roughly the same as the decision boundary is crossed.

\begin{figure}[ht]
\centering
\includegraphics[width = \textwidth]{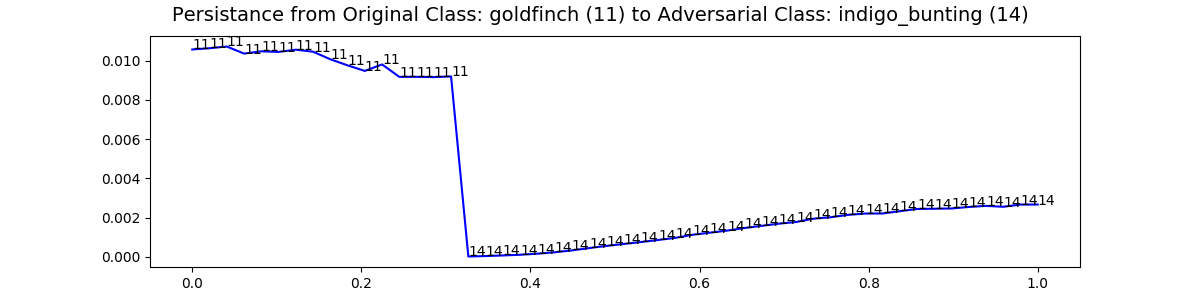}
\caption{The $0.7$-persistence of images along the straight line path from an image in class \texttt{goldfinch} (11) to an adversarial image generated with BIM in the class \texttt{indigo\_bunting} (14) on a vgg16 classifier. The classification of each image on the straight line is listed as a number so that it is possible to see the transition from one class to another. The vertical axis is $0.7$-persistence and the horizontal axis is progress towards the adversarial image.}\label{fig:persistent_interpimage}
\end{figure}

An aggregation of persistence for many randomly selected images from the \texttt{goldfinch} class in the validation set for Imagenet are presented in Table \ref{TAB:PersistenceAlexVGG}. 
\begin{table}[!ht]
\centering

\begin{tabular}{llll}
\toprule
Network/Method & Avg Dist & Persist (Nat) & Persist (Adv) \\
\midrule
alexnet (total) & 0.0194 & 0.0155 & 0.0049 \\ 
\:\: BIM        & 0.0188 & 0.0162 & 0.0050 \\ 
\:\: MIFGSM     & 0.0240 & 0.0159 & 0.0053 \\ 
\:\: PGD        & 0.0188 & 0.0162 & 0.0050 \\ 
\midrule
vgg16   (total) & 0.0154 & 0.0146 & 0.0011 \\ 
\:\: BIM        & 0.0181 & 0.0145 & 0.0012 \\ 
\:\: MIFGSM     & 0.0238 & 0.0149 & 0.0018 \\ 
\:\: PGD        & 0.0181 & 0.0145 & 0.0012 \\ 
\bottomrule
\end{tabular}
\caption{The $0.7$-persistence values for natural (Nat) and
  adversarial (Adv) images along with average distortion for
  adversarial images of alexnet and vgg16 for attacks generated with
  BIM, MIFGSM, and PGD on images from class \texttt{goldfinch}
  targeted toward other classes from the ILSVRC 2015 classification
  labels.} \label{TAB:PersistenceAlexVGG}%\label{table:attack_pers} 
\end{table}
For each image of a \texttt{goldfinch} and for each network of alexnet and vgg16, attacks were prepared to a variety of 28 randomly selected targets using a BIM, MIFGSM, PGD, FGSM, R+FGSM, and CW attack strategies. The successful attacks were aggregated and their $0.7$-persistences were computed using the Bracketing Algorithm along with the $0.7$-persistences of the original images from which each attack was generated. Each attack strategy had a slightly different mixture of which source image and attack target combinations resulted in successful attacks. The overall rates for each are listed, as well as particular results on the most successful attack strategies in our experiments, BIM, MIFGSM, and PGD. The results indicate that adversarial images generated for these networks (alexnet and vgg16) using these attacks were less persistent, and hence less stable, than natural images for the same models. 

\subsection{Decision Boundary Interpolation and Angle Measurement} \label{subsec:dbe}

\begin{figure}[!ht]
\centering\includegraphics[width=0.48\linewidth, trim=1.5cm 1.5cm 2cm 2cm, clip]{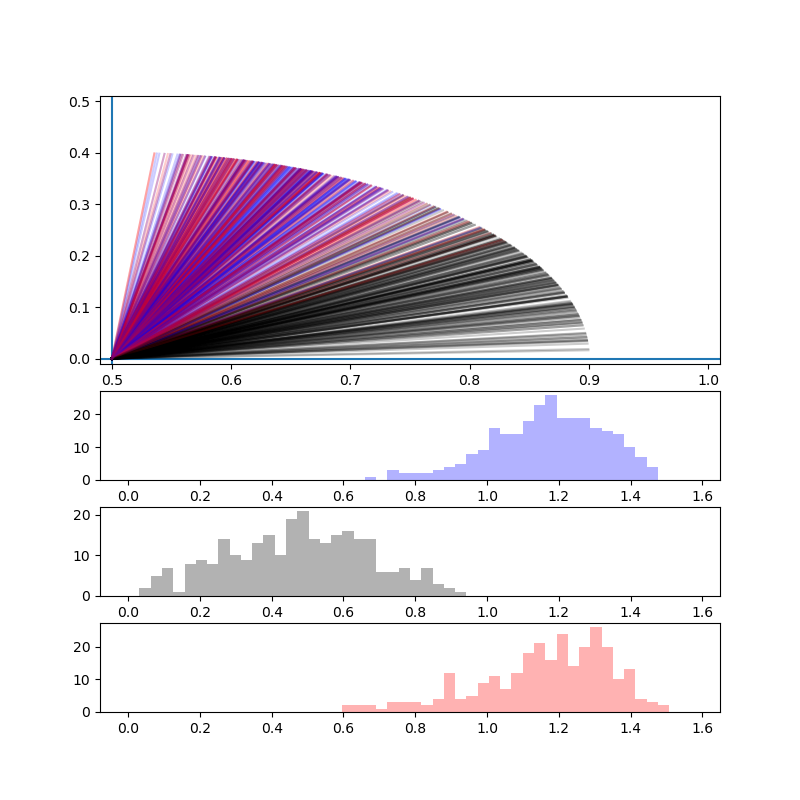}
\includegraphics[width=0.48\linewidth, trim=1.5cm 1.5cm 2cm 2cm, clip]{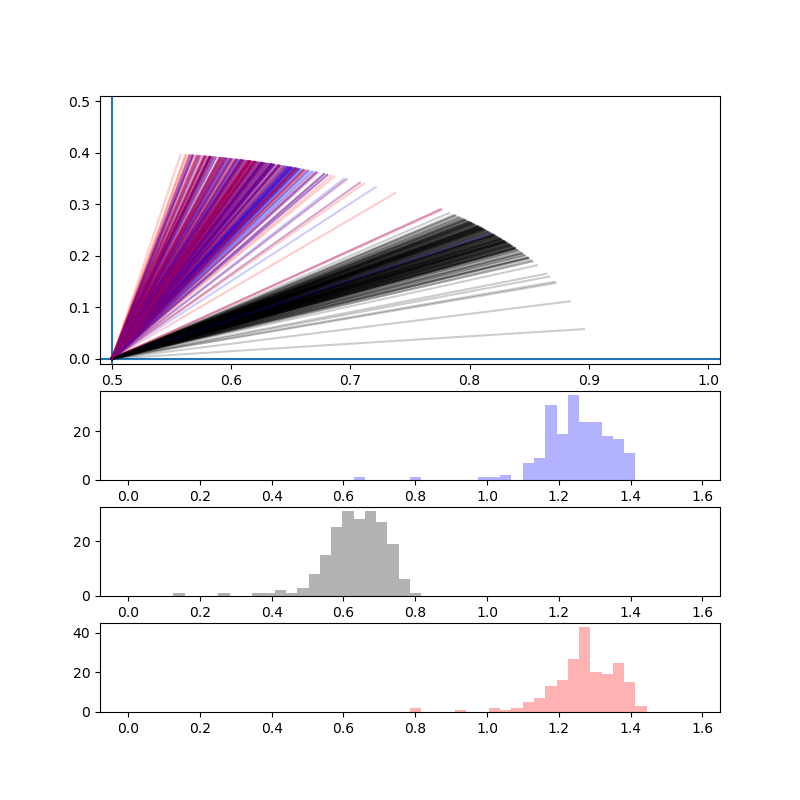}

\caption{Decision boundary incident angles between test to test
  interpolation and a computed normal vector to the decision boundary
  images (left) and between test and adversarial images
  (right). Angles (plotted Top) are referenced to decision boundary so
  $\pi/2$ radians (right limit of plots) corresponds with perfect
  orthogonality to decision boundary. Lines and histograms measure
  angles of training gradients (Top) linear interpolant (Middle) and
  adversarial gradients (Bottom). $x$ and $y$ axes are the axes of the
  unit-circle so angles can be compared. All angles are plotted in the
upper-right quadrant for brevity. The lower plots are all histograms
with their $y$ axes noting counts and their x-axes showing angles all
projected to the range from 0 to $\pi/2$.}
\label{fig:dba}
\end{figure}

In order to understand this sudden drop in persistence across the decision boundary observed in Figure ~\ref{fig:persistent_interpimage}, we will investigate incident angle of the interpolation with the decision boundary. In order to measure these angles, we must first interpolate along the decision boundary between two points. We will do this for pairs of test and test and pairs of test and adversary. In both cases, we will use a bracketing algorithm along the interpolation from candidate points to identify a point within machine-precision of the decision boundary $x_b$. 

Next, we will take 5000 samples from a Gaussian centered at this point with small standard deviation $\sigma = 10^{-6}$. Next, for each sample, we will perform an adversarial attack in order to produce a corresponding point on the opposite side of the decision boundary. Now for this new pair (sample and attacked sample), we will repeat the interpolation bracketing procedure in order to obtain the projection of this sample onto the decision boundary along the attack trajectory. Next, we will use singular value decomposition (SVD) on the differences between the projected samples and our decision boundary point $x_b$  to compute singular values and vectors from these projected samples. We will use the right singular vector corresponding with the smallest singular value as an approximation of a normal vector to the decision boundary at $x_b$. This point is difficult to compute due to degeneracy of SVD for small singular values, however in our tests, this value could be computed to a precision of 0.003. We will see that this level of precision exceeds that needed for the angles computed with respect to this normal vector sufficiently. 

From Figure~\ref{fig:dba} we notice that neither training gradients
nor adversarial gradients are orthogonal to the decision
boundary. From a theory perspective, this is possible because this
problem has more than 2 classes, so that the decision boundary
includes $(0.34, 0.34, 0.32)$ and $(0.4, 0.4, 0.2)$. That is to say
that the level set definition of the decision boundary has degrees of
freedom that do not require orthogonality of gradients. More
interestingly, both natural and adversarial linear interpolants tend
to cross at acute angles with respect to the decision boundary, with
adversarial attacks tending to be closer to orthogonal. This suggests
that obliqueness of the decision boundary with respect to test points
may be related to adversarial vulnerability. We will leverage this understanding with manifold alignment to see if constraining gradients to a lower dimensional manifold, and thus increasing orthogonality of gradients will increase robustness. 

\subsection{Manifold Alignment on MNIST via PCA} \label{subsec:mae}

In order to provide an empirical measure of alignment, we first require a well defined image manifold.
The task of discovering the true structure of \textit{k}-dimensional manifolds in $\mathds{R}^d$ given a set of points sampled on the manifold has been studied previously \citep{khoury2018}.
Many algorithms produce solutions which are provably accurate under data density constraints.
Unfortunately, these algorithms have difficulty extending to domains with large $d$ due to the curse of dimensionality.
Our solution to this fundamental problem is to sidestep it entirely by redefining our dataset.
We begin by projecting our data onto a well known low dimensional manifold, which we can then measure with certainty.

We first fit a PCA model on all training data, using $k$ components
for each class to form a component matrix $W$, where $k << d$.
Given the original dataset $X$, we create a new dataset $X_{\mathcal{M}} := \{x \times \textbf{W}^T \times \textbf{W} : x \in X \}$.
We will refer to this set of component vectors as $\textbf{W}$.
Because the rank of the linear transformation matrix, $k$, is defined lower than the dimension of the input space, $d$, this creates a dataset which lies on a linear subspace of $\mathds{R}^d$.
This subspace is defined by the span of $X \times \textbf{W}^T$ and any vector in $\mathds{R}^d$ can be projected onto it.
Any data point drawn from $\{z \times \textbf{W}^T : z \in \mathds{R}^k \}$ is considered a valid datapoint.
This gives us a continuous linear subspace which can be used as a data manifold.

Given that it is our goal to study the simplest possible case, we chose MNIST as the dataset to be projected and selected $k = 28$ components.
We refer to this new dataset as Projected MNIST (PMNIST).
The true rank of PMNIST is lower than that of the original MNIST data, meaning there was information lost in this projection.
The remaining information we found is sufficient to achieve 92\% accuracy using a baseline Multi-layer Perceptron (MLP), and the resulting images retain their semantic properties as shown in Figure \ref{fig:perception}.

\begin{figure*}[ht]
    \centering
    \includegraphics[width=0.25\linewidth]{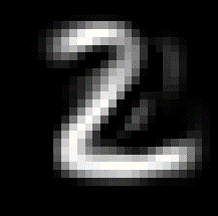}
    \includegraphics[width=0.25\linewidth]{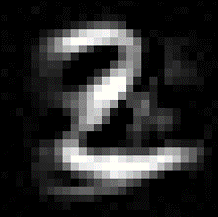}
    \includegraphics[width=0.25\linewidth]{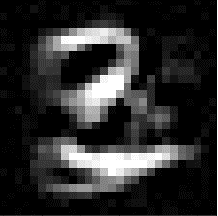}
    \caption{Visual example of manifold optimized model transforming 2 into 3. Original PMNIST image on left, center image is center point between original and attacked, on right is the attacked image. Transformation performed using PGD using the $l_\infty$ norm. Visual evidence of manifold alignment is often subjective and difficult to quantify. This example is provided as a baseline to substantiate our claim that our empirical measurements of alignment are valid.}
    \label{fig:perception}
\end{figure*}

\subsection{Manifold Aligned Gradients} \label{subsec:ma}

Component vectors extracted from the original dataset are used to project gradient examples onto our pre-defined image manifold.

Given a gradient example $\nabla_x = \frac{\partial f_\theta(x, y)}{\partial x}$ where $f_\theta$ represents a neural network parameterized by weights $\theta$, $\nabla_x$ is transformed using the coefficient vectors \textbf{W}.

\begin{equation}
    \rho_x = \nabla_x \times \textbf{W}^T \times \textbf{W}    
\end{equation}
The projection of the original vector onto this new transformed vector will be referred to as $P_{\mathcal{M}}$.
\begin{equation}
    P_{\mathcal{M}}(\nabla_x) = \dfrac{\nabla_x \cdot \rho_x}{||\rho_x||_2} \cdot \dfrac{\rho_x}{||\rho_x||_2}
\end{equation}
The ratio of norms of this projection gives a metric of manifold alignment.
\begin{equation}
    \frac{|| \nabla_x || }{||P_{\mathcal{M}}(\nabla_x )||}.
  \label{equ:ratio}
\end{equation}
This gives us a way of measuring the ratio between on-manifold and off-manifold components of the gradient.
In addition to this loss formula, both cosine similarity and the vector rejection length were tested, but the norm ratio we found to be the most stable in training.
We hypothesize that this is due to these other metrics becoming less convex, and thus less stable, when far from minima.
It is unknown whether this result is true in general or if the specific dataset and optimization strategy used influenced which loss was most stable.
We use this measure as both a metric and a loss, allowing us to optimize the following objective.
\begin{equation}
  \mathds{E}_{(x,y) \sim \mathcal{D}} \left[ L(\theta, x,y)  + \alpha \frac{|| \nabla_x || }{||P_{\mathcal{M}}(\nabla_x )||} \right]
  \label{equ:loss}
\end{equation}

Where $L(\theta, x, y)$ represents our classification loss term and $\alpha$ is a hyper parameter determining the weight of the manifold loss term.

\subsection{Manifold Alignment Robustness Results}

All models were two layer MLPs with 1568 nodes in each hidden layer.
The hidden layer size was chosen as twice the input size.
This arrangement was chosen to maintain the simplest possible case.

Two types of attacks were leveraged in this study: fast gradient sign
method (FGSM) \citep{goodfellow_explaining_2014} which performs a
single step based on the sign of an adversarial gradient for each
input and projected gradient descent (PGD) which performs gradient
descent on input data using adversarial gradients in order to produce
adversarial attacks \citep{madry_towards_2017}.
A total of four models were trained and evaluated on these attacks: Baseline, Robust, Manifold and Manifold Robust.
All models, including the baseline, were trained on PMNIST (a fixed
permutation is applied to the training and test images of the MNIST dataset).
``Robust" in our case refers to models trained with new adversarial
examples labeled for their class \emph{before} perturbation during
each epoch consistent with ~\citet{tramer2019adversarial}. All robust
models were trained using the $l_\infty$ norm and a maximum
perturbation parameter of $\epsilon = 0.1$. Manifold Robust refers to
both optimizing our manifold objective and robust training simultaneously.

\begin{figure}[ht]
\begin{center}
    % \fbox{\rule[-.5cm]{0cm}{4cm} \rule[-.5cm]{4cm}{0cm}}

    \includegraphics[width=0.45\linewidth]{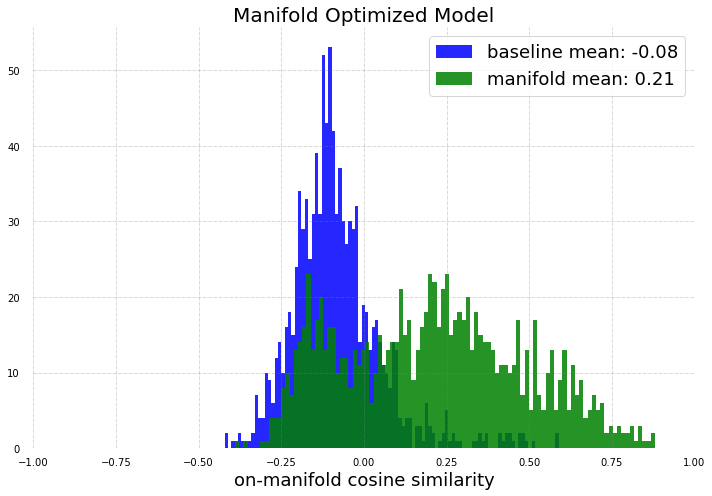}\includegraphics[width=0.45\linewidth]{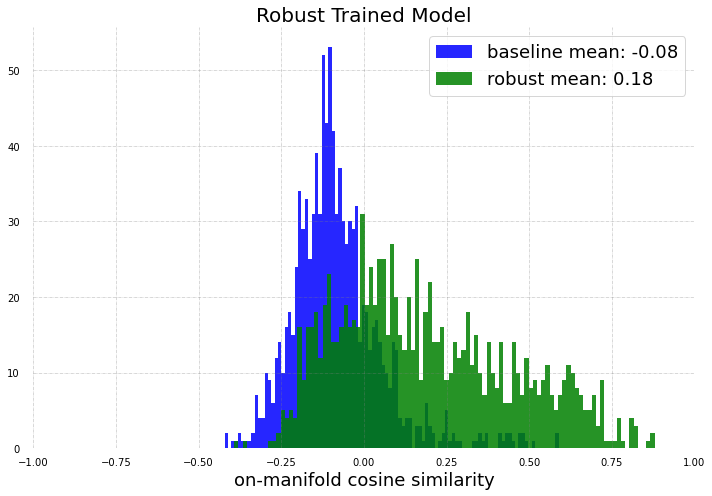}
\end{center}
    \caption{Comparison of on-manifold components between baseline network, robust trained models, and manifold optimized models. Large values indicate higher similarity to the manifold. $Y$-axes for both plots are histogram counts. Both robust and manifold optimized models are more 'on-manifold' than the baseline, with adversarial training being slightly less so.}
    \label{fig:hist_cosine}
\end{figure}

\begin{figure}[ht]
    \centering
    \includegraphics[width=0.45\linewidth]{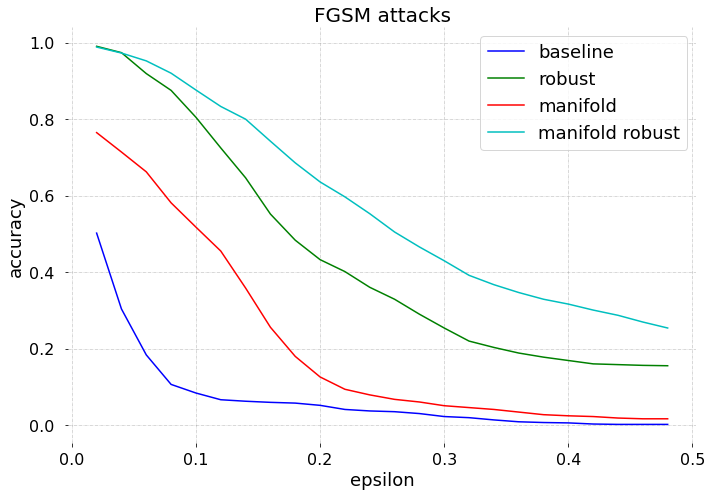}\includegraphics[width=0.45\linewidth]{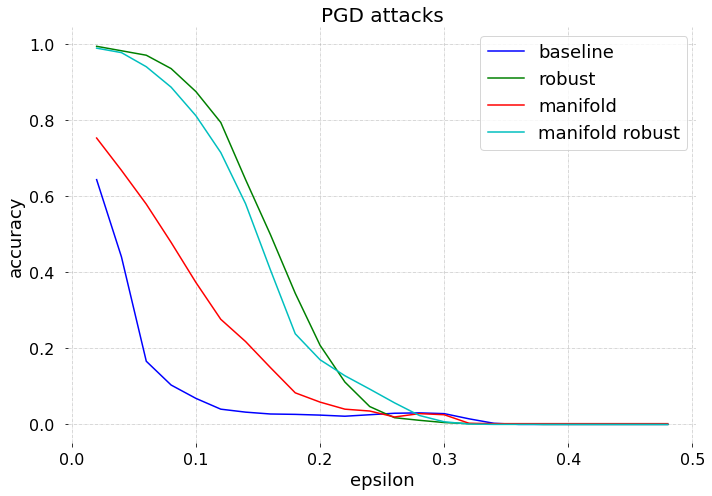}
    \caption{Comparison of adversarial robustness for PMNIST models under various training conditions. Attacks are prepared using a range of a distortion parameter epsilon. For FGSM, the sign of the gradient is multiplied by each epsilon. For PGD, epsilon is determined by a weight on the $l_2$ norm term of the adversarial loss function. Many variations of the $l_2$ weight are performed, and then they are aggregated and the distance of each perturbation is plotted as epsilon. For both FGSM and PGD, we see a slight increase in robustness from using manifold optimization. Adversarial training still improves performance significantly more than manifold optimization. Another observation to note is that when both the manifold, and adversarial objective were optimized, increased robustness against FGSM attacks was observed. All robust models were trained using the $l_\infty$ norm at epsilon = 0.1.}
    \label{fig:model_robustness}
\end{figure}

Figure \ref{fig:hist_cosine} shows the cosine similarity of the
gradient and its projection onto the reduced space $W$ on the testing set of PMNIST for both the Manifold model and Robust model.
Higher values (closer to 1) indicate the model is more aligned with the manifold.
Both robust and MAG models here are shown to be more on manifold than the Baseline.
This demonstrates that our metric for alignment is being optimized as a consequence of adversarial training.

Figure \ref{fig:model_robustness} shows the adversarial robustness of each model.
In both cases, aligning the model to the manifold shows an increase in
accuracy of classification for adversarial images over the baseline.
However, we do not consider the performance boost against PGD to be significant enough to call these models robust against PGD attacks.
Another point of interest that while using both our manifold alignment metric and adversarial training, we see an even greater improvement against FGSM attacks.

The fact that this performance increase is not shared by PGD training may indicate a relationship between these methods.
Our current hypothesis is that a linear representation of the image manifold is sufficient to defend against linear attacks such as FGSM, but cannot defend against a non-linear adversary.

\section{Conclusion}

In order to better understand the observed tendency for points near natural data to be classified similarly and points near
adversarial examples to be classified differently, we defined a notion of $(\gamma,\sigma)$-stability which is easily estimated by Monte Carlo sampling. For any data point $x$, we then define the $\gamma$-persistence to to be the smallest $\sigma_\gamma$ such that the probability of similarly classified data is at least $\gamma$ when sampling from Gaussian distributions with mean $x$ and standard deviation less than $\sigma_\gamma$. The persistence value can be quickly estimated by a Bracketing Algorithm. These two measures were considered with regard to both the MNIST and ImageNet datasets and with respect to a variety of classifiers and adversarial attacks. We found that adversarial examples were much less stable than natural examples in that the $0.7$-persistence for natural data was usually significantly larger than the $0.7$-persistence for adversarial examples. We also saw that the dropoff of the persistence tends to happen precisely near the decision boundary. Each of these observations is strong evidence toward the hypothesis that adversarial examples arise inside cones or high curvature regions in the adversarial class, whereas natural images lie outside such regions.

We also found that often the most likely class for perturbations of an adversarial examples is a class other than the class of the original natural example used to generate the adversarial example; instead, some other background class is favored. In addition, we found that some adversarial examples may be more stable than others, and a more detailed probing using the concept of $(\gamma,\sigma)$-stability and the $\gamma$-persistence statistic may be able to help with a more nuanced understanding of the geometry and curvature of the decision boundary. Although not pursued here, the observations and statistics used in this paper could potentially be used to develop methods to detect adversarial examples as in \citep{crecchi2019,frosst2018,hosseini2019odds,Lee2018ASU,qin2020,roth19aodds} and others. As with other methods of detection, this may be susceptible to adaptive attacks as discussed by ~\citet{tramer2020adaptive}. Critically, our results summarized in Table~\ref{table1} indicate that for more regularized and smaller models, model accuracy is not necessarily a good indicator of resistance to adversarial attacks. We believe that measurement of distortion and even a metric like persistence should be a core part of adversarial robustness evaluation augmenting standard practice as set forward by ~\citet{carlini2019}. 

For the future, we have made several observations: We found that some adversarial examples may be more stable than others. More detailed probing using the concept of $(\gamma,\sigma)$-stability and the $\gamma$-persistence along linear interpolation between natural images and between natural and adversarial images reveals sharp drops in persistence. Sharp drops in persistence correspond with oblique angles of incidence between linear interpolation vectors and the decision boundary learned by neural networks. Combining these observations, we can form a conjecture: Adversarial examples appear to exist near regions surrounded by negatively curved structures bounded by decision surfaces with relatively small angles relative to linear interpolation among training and testing data. This conjecture compares with the dimpled manifold hypothesis ~\citep{shamir2021dimpled}, however our techniques provide geometric information that allows us to gain a more detailed analysis of this region than in that work. In addition, our analysis of gradient alignment with manifolds reinforces the notion that the obliqueness we observe may be a property which can be isolated and trained out of neural networks to some extent. Future work should focus on refining this conjecture with further tools to complete the spatial and mathematical picture surrounding adversarial examples.

\nocite{langley00}

\begin{ack}
This material is based upon work supported by the Department of Energy (National Nuclear Security Administration Minority Serving Institution Partnership Program's CONNECT - the COnsortium on Nuclear sECurity Technologies) DE-NA0004107.
This report was prepared as an account of work sponsored by an agency of the United States Government.
Neither the United States Government nor any agency thereof, nor any of their employees, makes any warranty, express or implied, or assumes any legal liability or responsibility for the accuracy, completeness, or usefulness of any information, apparatus, product, or process disclosed, or represents that its use would not infringe privately owned rights. The views and opinions of authors expressed herein do not necessarily state or reflect those of the United States Government or any agency thereof.

We would like to acknowledge funding from NSF TRIPODS Award Number 1740858 and NSF RTG Applied Mathematics and Statistics for Data-Driven Discovery Award Number 1937229. The manifold alignment portion of this research was supported by LANL’s Laboratory Directed Research and Development (LDRD) program under project number 20210043DR.
\end{ack}
%Kevin Lin, Mingwei
%Grant funding: BB/DG/KH/CS NSF CCF 1740858 
%Grant funding: BB/DG: NSF DMS 1937229
%Grant funding: DG: NSF DMS 1760538
\bibliographystyle{abbrvnat}
\bibliography{main}

%%%%%%%%%%%%%%%%%%%%%%%%%%%%%%%%%%%%%%%%%%%%%%%%%%%%%%%%%%%%

\newpage
\appendix

\section{Bracketing Algorithm}\label{sec:bracketing}
The Bracketing Algorithm is a way to determine persistance of an image with respect to a given classifier, typically a DNN. The algorithm was implemented in Python for the experiments presented. The \textproc{rangefinder} function is not strictly necessary, in that one could directly specify values of $\sigma_{\min}$ and $\sigma_{\max}$, but we include it here so that the code could be automated by a user if so desired.

\begin{algorithm} [h!]
\begin{algorithmic}
\Function{bracketing}{image, classifier ($\CC$), numSamples, $\gamma$, maxSteps, precision}

\State $[\sigma_{\min},\sigma_{\max}] = $\textproc{rangefinder}(image, $\CC$, numSamples, $\gamma$)
\State count $=1$
\While{count$<$maxSteps}
\State $\sigma = \frac{\sigma_{\min}+\sigma_{\max}}{2}$
\State $\gamma_{\textnormal{new}} =$ \textproc{compute\_persistence}($\sigma$, image, numSamples, $\CC$)
\If{$|\gamma_{\textnormal{new}}-\gamma|<$precision}
\State \textbf{return} $\sigma$ %$\gamma_{\textnormal{new}}$
\ElsIf{$\gamma_{\textnormal{new}}>\gamma$}
\State $\sigma_{\min} = \sigma$
\Else
\State $\sigma_{\max} = \sigma$
\EndIf
\State count = count + 1
\EndWhile

\Return $\sigma$
\EndFunction

\\
\Function{rangefinder}{image, $\CC$, numSamples, $\gamma$}
\State $\sigma_{\min}=.5$,\;\; $\sigma_{\max}=1.5$
\State $\gamma_1 =$ \textproc{compute\_persistence}($\sigma_{\min}$, image, numSamples, $\CC$)
\State $\gamma_2 =$ \textproc{compute\_persistence}($\sigma_{\max}$, image, numSamples, $\CC$)
\While{$\gamma_1<\gamma$ \textbf{or} $\gamma_2>\gamma$}
\If{$\gamma_1<\gamma$}
\State $\sigma_{\min} = .5\sigma_{\min}$
\State $\gamma_1 =$ \textproc{compute\_persistence}($\sigma_{\min}$, image, numSamples, $\CC$)
\EndIf
\If{$\gamma_2>\gamma$}
\State $\sigma_{\max} = 2\sigma_{\max}$
\State $\gamma_2 =$ \textproc{compute\_persistence}($\sigma_{\max}$, image, numSamples, $\CC$)
\EndIf
\EndWhile

\Return{$[\sigma_{\min}, \sigma_{\max}]$}
%\sigma_{\min},\sigma_{\max}$}
\EndFunction

\\
\Function{compute\_persistence}{$\sigma$, image, numSamples, $\CC$}
\State sample = $N(\textnormal{image},\sigma^2I,$numSamples)
\State $\gamma_{\textnormal{est}} = \frac{|\{\CC(\textnormal{sample})=\CC(\textnormal{image})\}|}{\textnormal{numSamples}}$

\Return{$\gamma_{\textnormal{est}}$}
\EndFunction
\end{algorithmic}
\caption{Bracketing algorithm for computing $\gamma$-persistence}\label{bracketing}
\end{algorithm}

\section{Convolutional neural networks used} \label{appendix:CNNs}
In Table \ref{table1} we reported results on varying complexity convolutional neural networks. These networks consist of a composition of convolutional layers followed by a maxpool and fully connected layers. 
The details of the network layers are described in Table \ref{tab:CNN} where Ch is the number of channels in the convolutional components. %\todo{[DG]: We need to simplify to the tables without the truncation}

%\vspace{.4cm}
\begin{table}[pt]
\centering
\caption{Structure of the CNNs C-Ch used in Table \ref{table1}}
\label{tab:CNN}
\begin{tabular}{llllll}
\toprule
     Layer & Type & Channels & Kernel & Stride & Output Shape \\
\midrule
     0 & Image & 1 & NA & NA & $(1, 28, 28)$ \\
     1 & Conv & Ch & $(5,5)$& $(1,1)$& $(\textnormal{Ch}, 24, 24)$\\
     2 & Conv & Ch & $(5,5)$& $(1,1)$& $(\textnormal{Ch}, 20, 20)$\\
     3 & Conv & Ch & $(5,5)$& $(1,1)$& $(\textnormal{Ch}, 16, 16)$\\
     4 & Conv & Ch & $(5,5)$& $(1,1)$& $(\textnormal{Ch}, 12, 12)$\\
     5 & Max Pool & Ch & $(2, 2)$ & $(2, 2)$& $(\textnormal{Ch}, 6, 6)$ \\
     %6 & Trunc & 1 & NA & NA & $Ch$ \\
     7 & FC & $(\textnormal{Ch}\cdot 6 \cdot 6, 256)$ & NA & NA & 256 \\
     8 & FC & $(256, 10)$ & NA & NA & 10 \\
     \bottomrule
\end{tabular}
\end{table}
 
%[DG: What changes if the number of channels changes? Can we give a general form.]
%We denote such a network as ``C-Ch-$k$,'' where C reflects that this is a convolutional DNN, Ch denotes the number of channels and $k$ denotes the width of the smallest level, for instance, ``C-128-2''.

\section{Additional Figures}
In this section we provide additional figures to demonstrate some of the experiments from the paper.

%\subsection{Figures interpolating between natural and adversarial examples}
%In this section we further investigate what happens on the straight line from a natural example to an adversarial example.

\subsection{Additional figures from MNIST}
In Figure \ref{fig:mnistadv} we begin with an image of a \texttt{1} and generate adversarial examples to the networks described in Section \ref{sec:mnist} via IGSM targeted at each class \texttt{2} through \texttt{9}; plotted are the counts of output classifications by the DNN from samples from Gaussian distributions with increasing standard deviation; this complements Figure \ref{fgsmo} in the main text. Note that the prevalence of the adversarial class falls off quickly in all cases, though the rate is different for different choices of target class.
\begin{figure}[!htb]
    \centering
    \includegraphics[width=.49\textwidth]{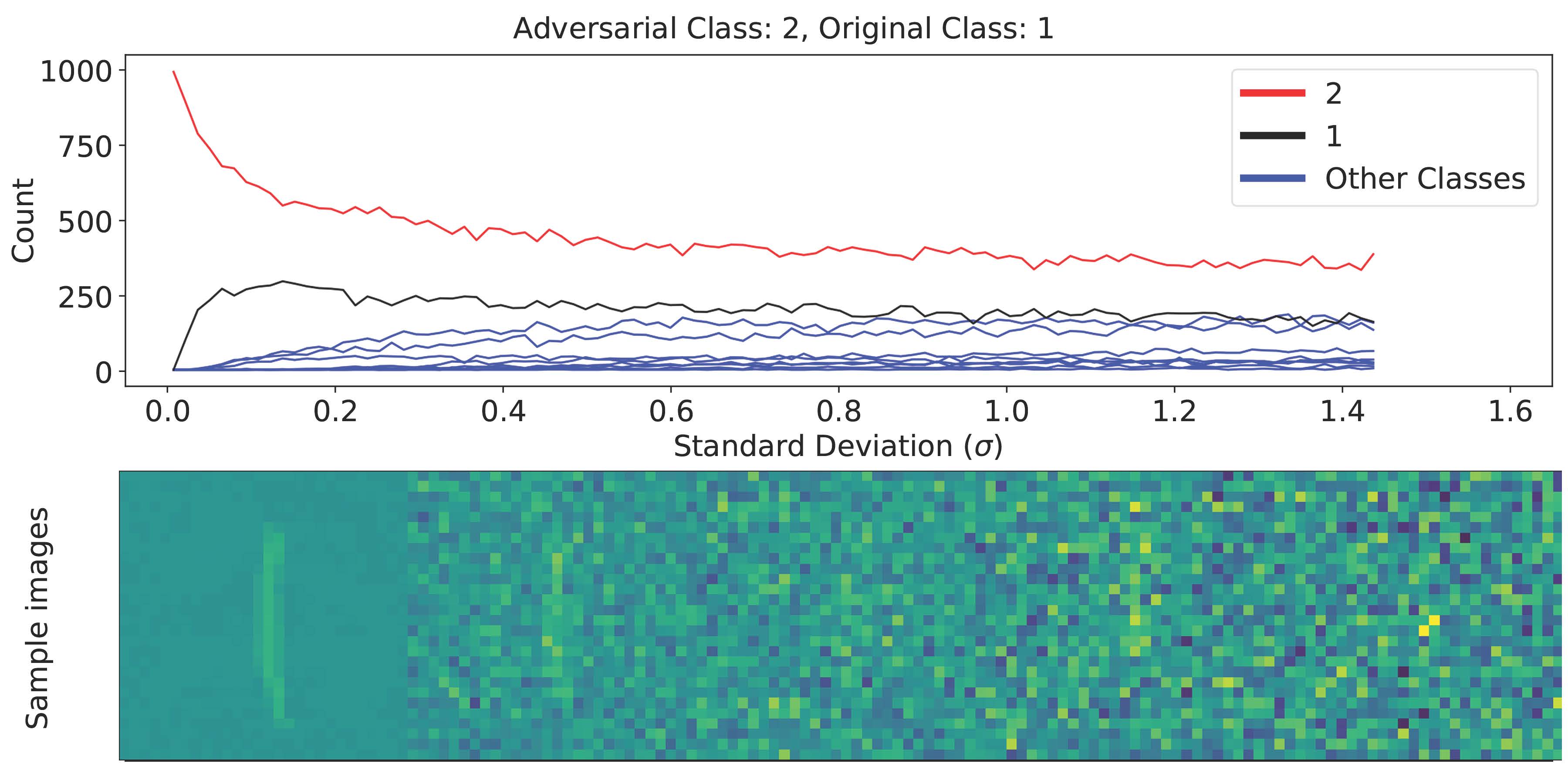}
    \includegraphics[width=.49\textwidth]{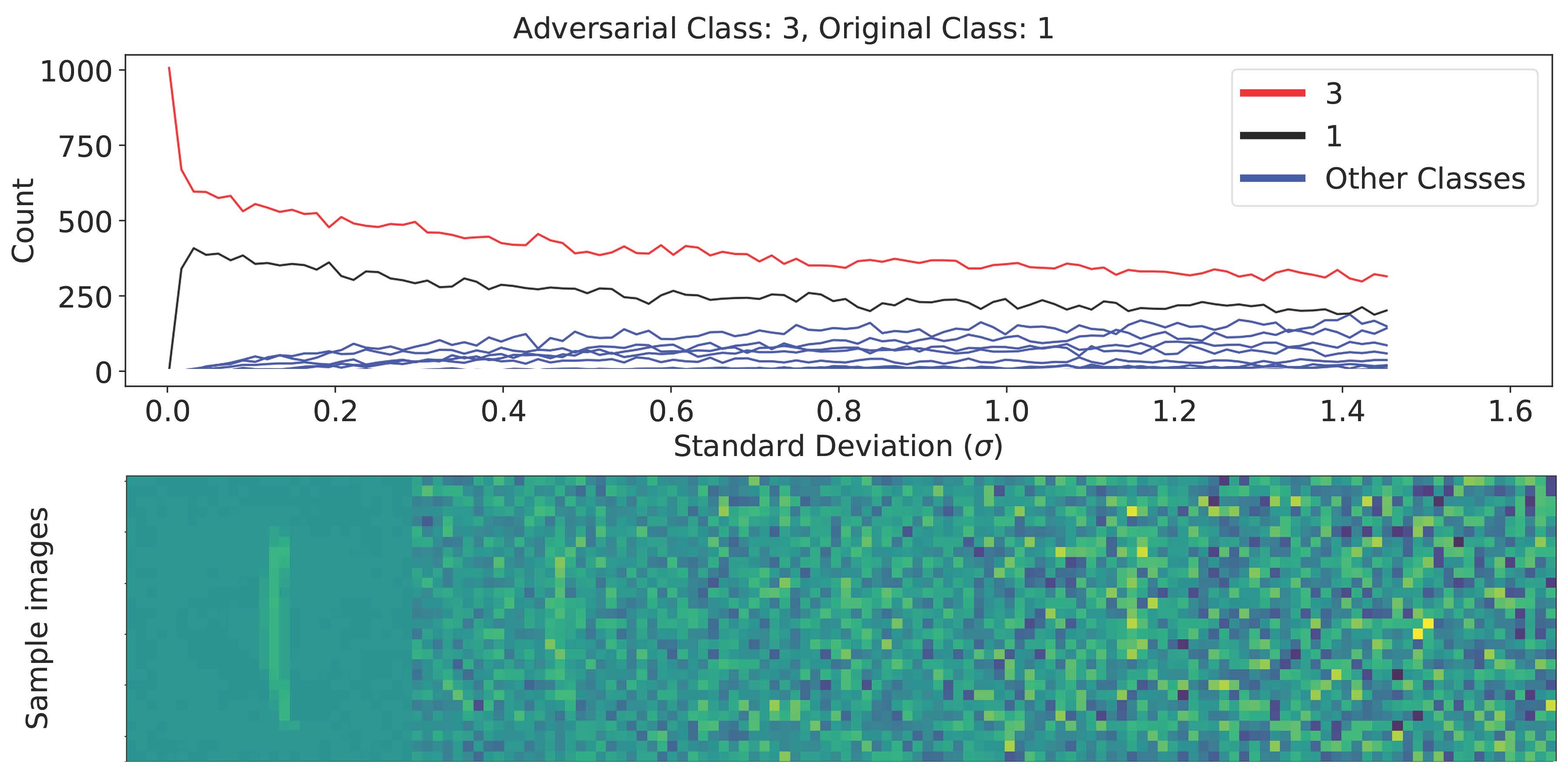}
    \includegraphics[width=.49\textwidth]{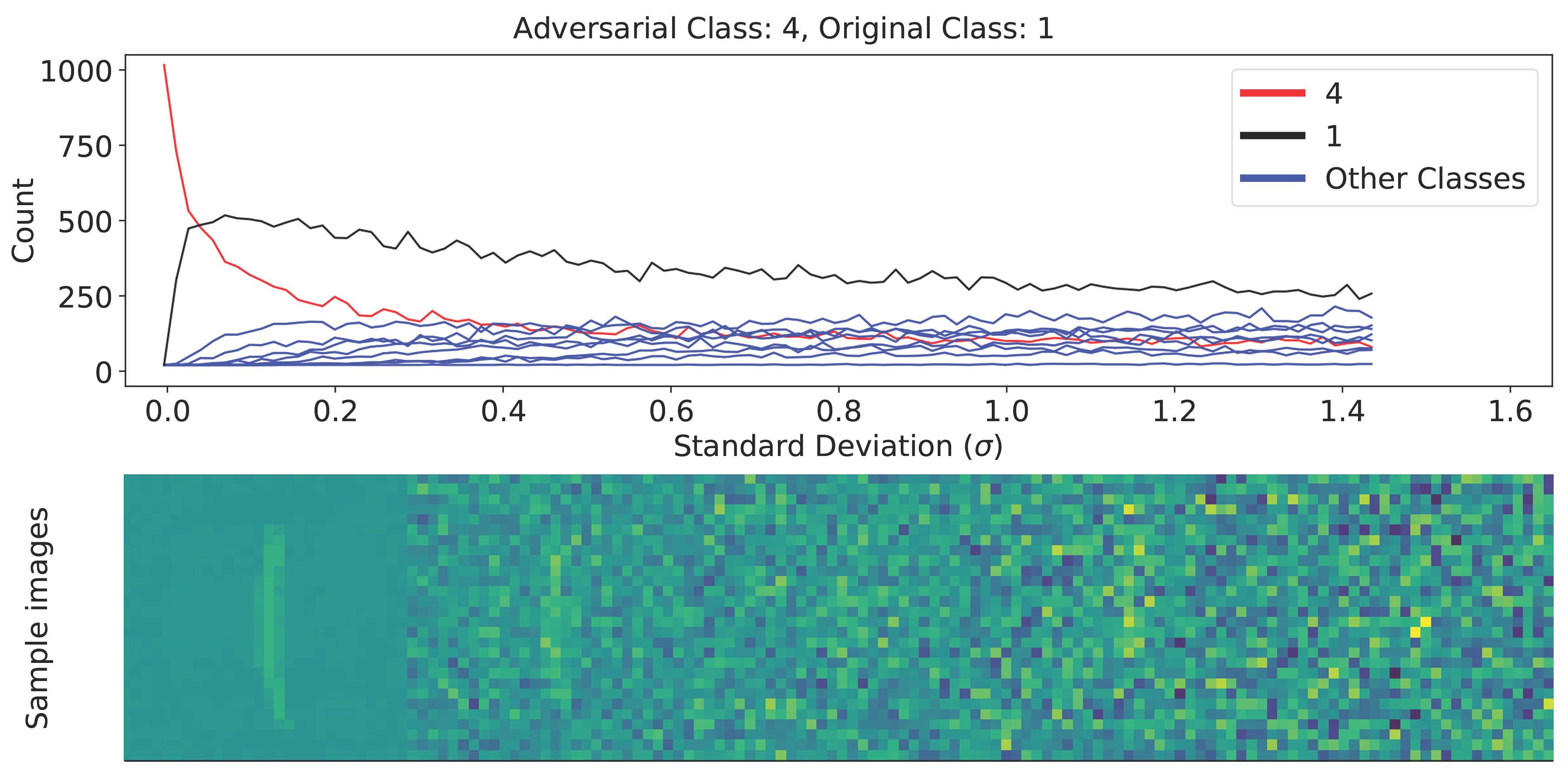}
    \includegraphics[width=.49\textwidth]{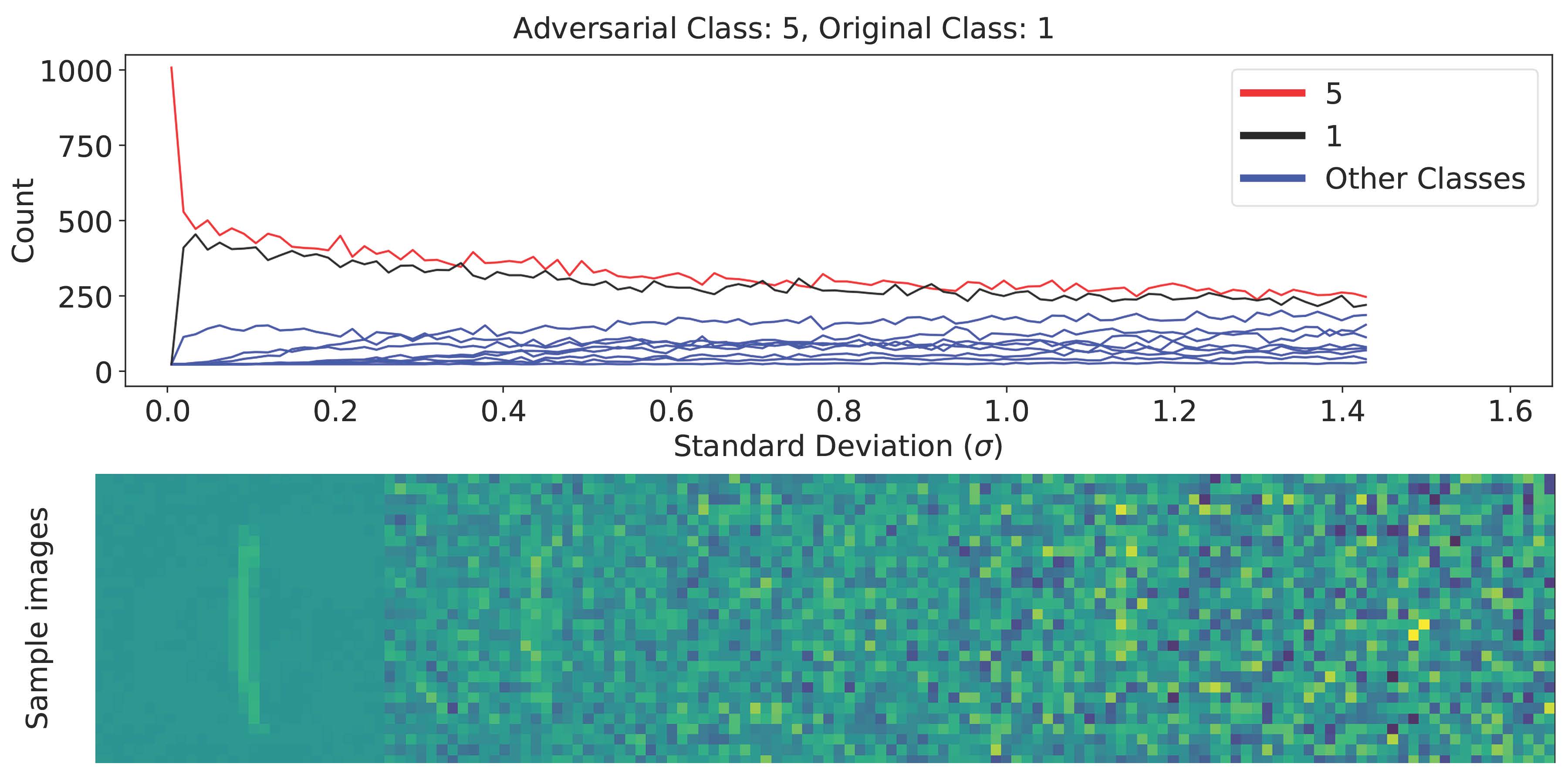}
    \includegraphics[width=.49\textwidth]{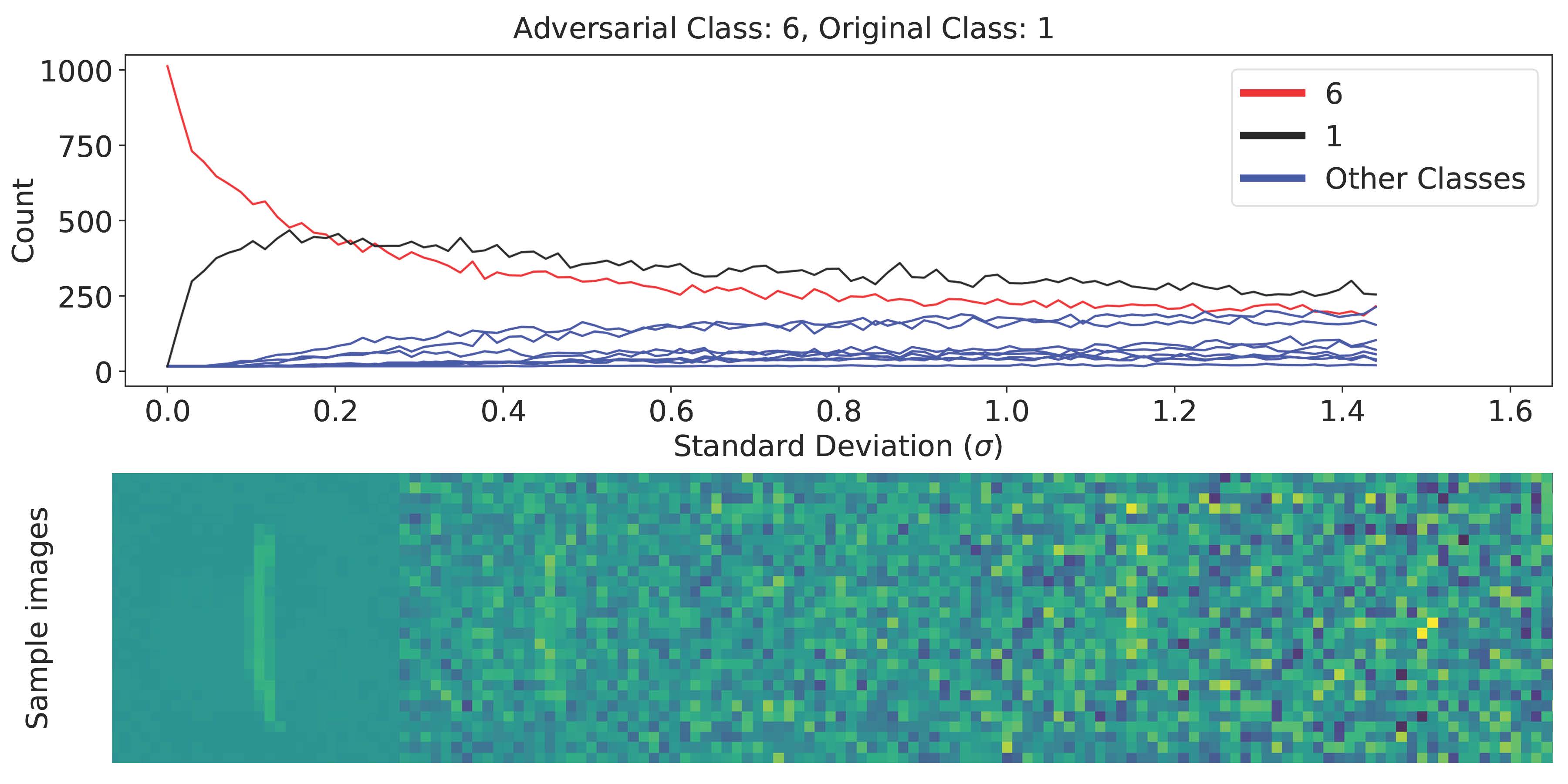}
    \includegraphics[width=.49\textwidth]{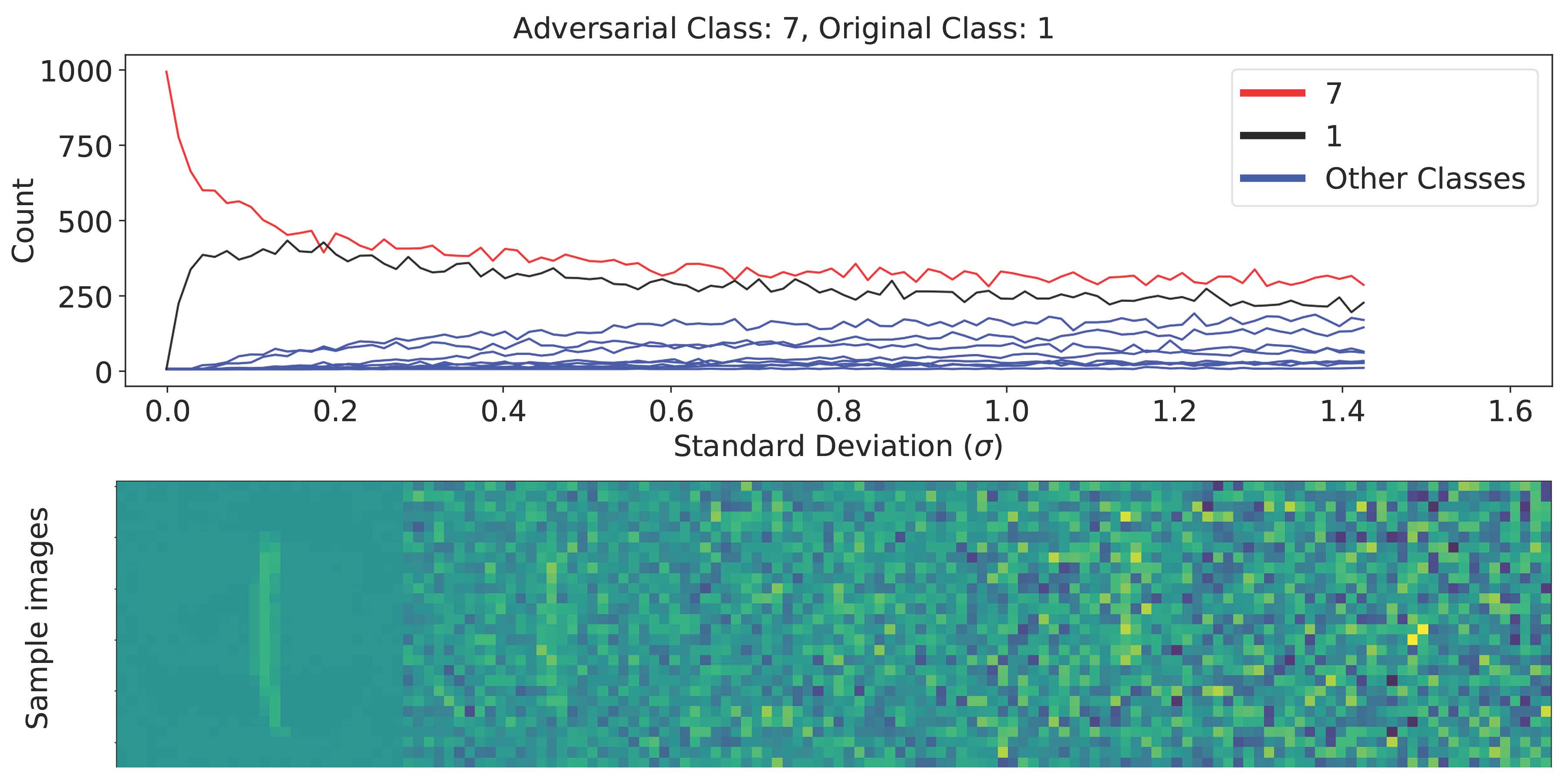}
    \includegraphics[width=.49\textwidth]{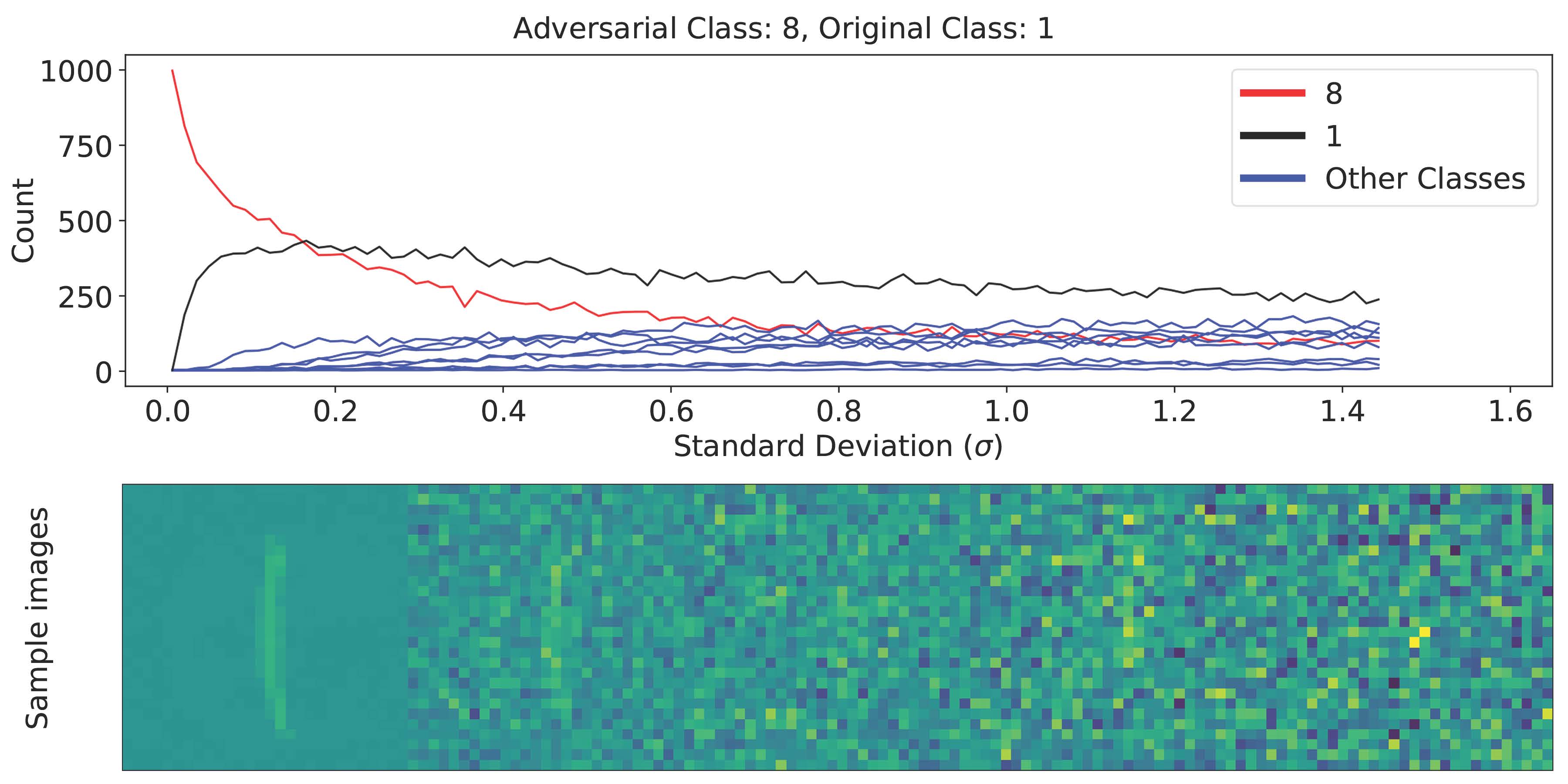}
    \includegraphics[width=.49\textwidth]{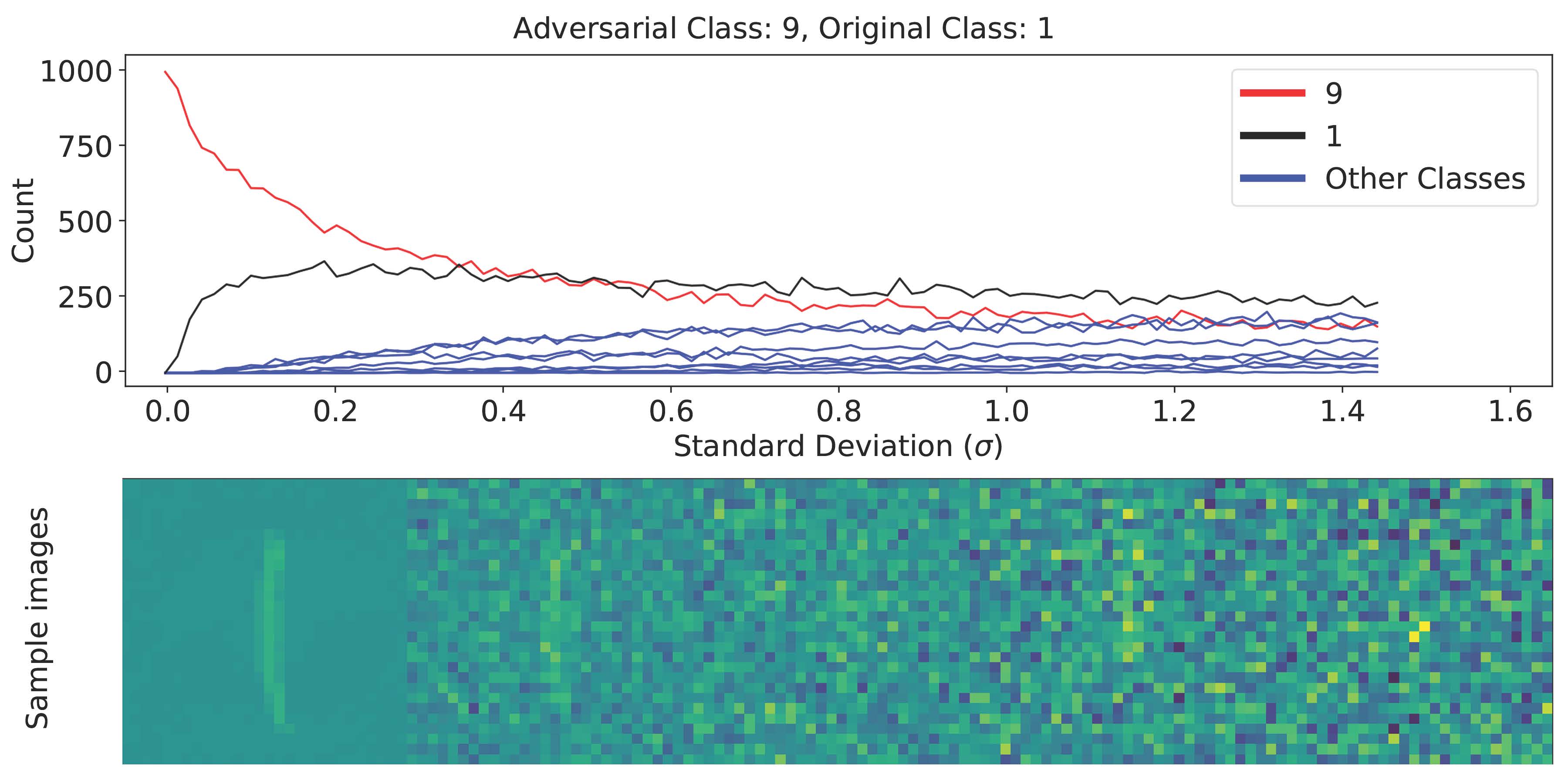}
    \caption{Frequency of each class in Gaussian samples with increasing standard deviations around adversarial attacks of an image of a \texttt{1} targeted at classes \texttt{2} through \texttt{9} on a DNN classifier generated using IGSM. The adversarial class is shown as a red curve. The natural image class (\texttt{1}) is shown in black. Bottoms show example sample images at different standard deviations.}
    \label{fig:mnistadv}
\end{figure}

We also show histograms corresponding to those in Figure \ref{fig:IGSMpersistenceMNIST} and the networks from Table \ref{table1}.  As before, for each image, we used IGSM to generate 9 adversarial examples (one for each target class) yielding a total of 1800 adversarial examples. In addition, we randomly sampled 1800 natural MNIST images. For each of the 3600 images, we computed $0.7$-persistence. In Figure \ref{fig:FC10}, we see histograms of these persistence values for the small fully connected networks with increasing levels of regularization. In each case, the test accuracy is relatively low and distortion relatively high. It should be noted that these high-distortion attacks against models with few effective parameters were inherently very stable -- resulting in most of the ``adversarial'' images in these sets having higher persistence than natural images. This suggests a lack of the sharp conical regions which appear to characterize adversarial examples generated against more complicated models.  In Figure \ref{fig:FC100200} we see the larger fully connected networks from Table \ref{table1} and in Figure \ref{fig:CNNs} we see some of the convolutional neural networks from Table \ref{table1}. 

\begin{figure}[!htb]
\centering
\includegraphics[trim=200 80 100 100, clip,width=.32\textwidth]{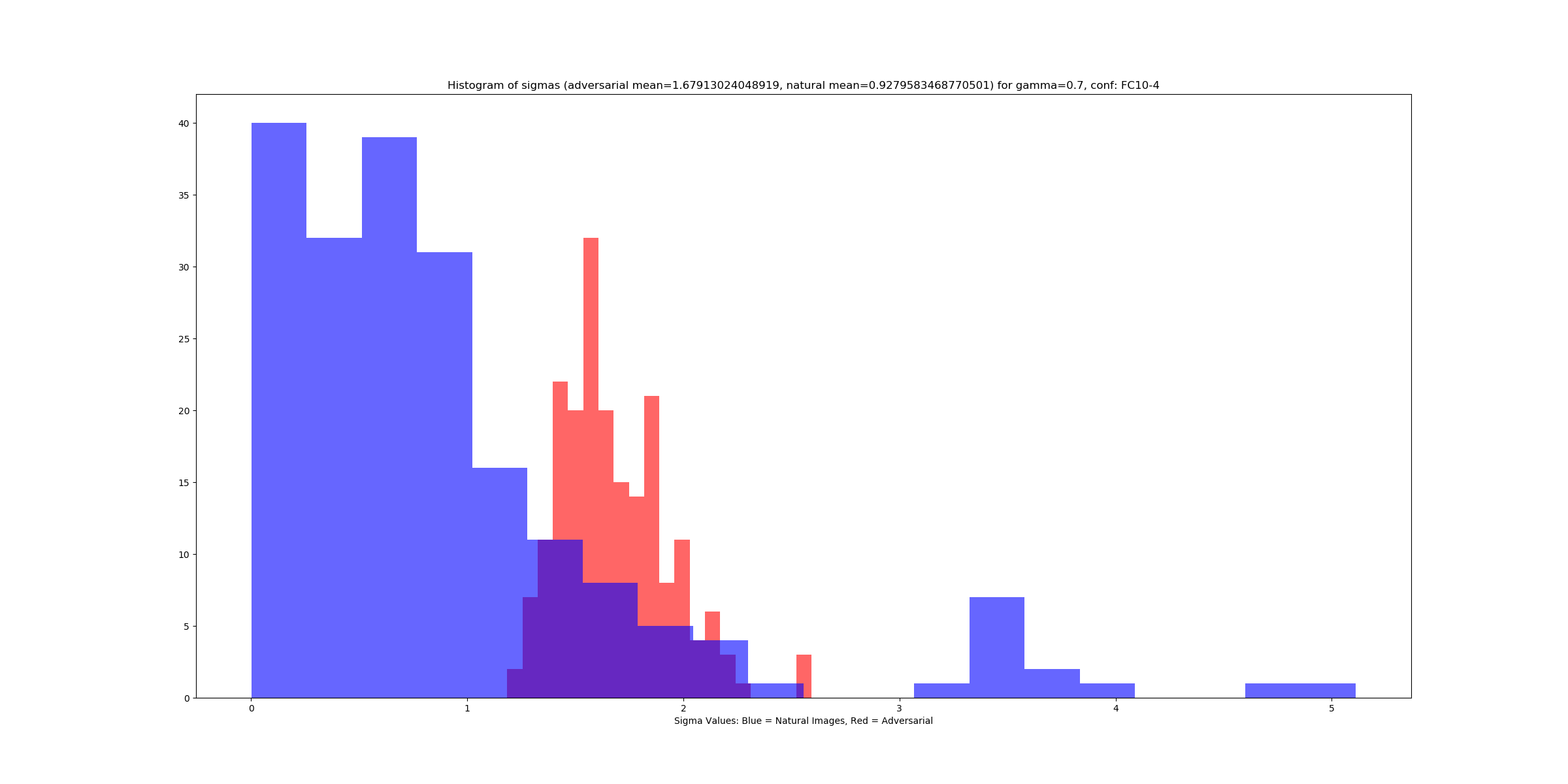}
\includegraphics[trim=200 80 100 100, clip,width=.32\textwidth]{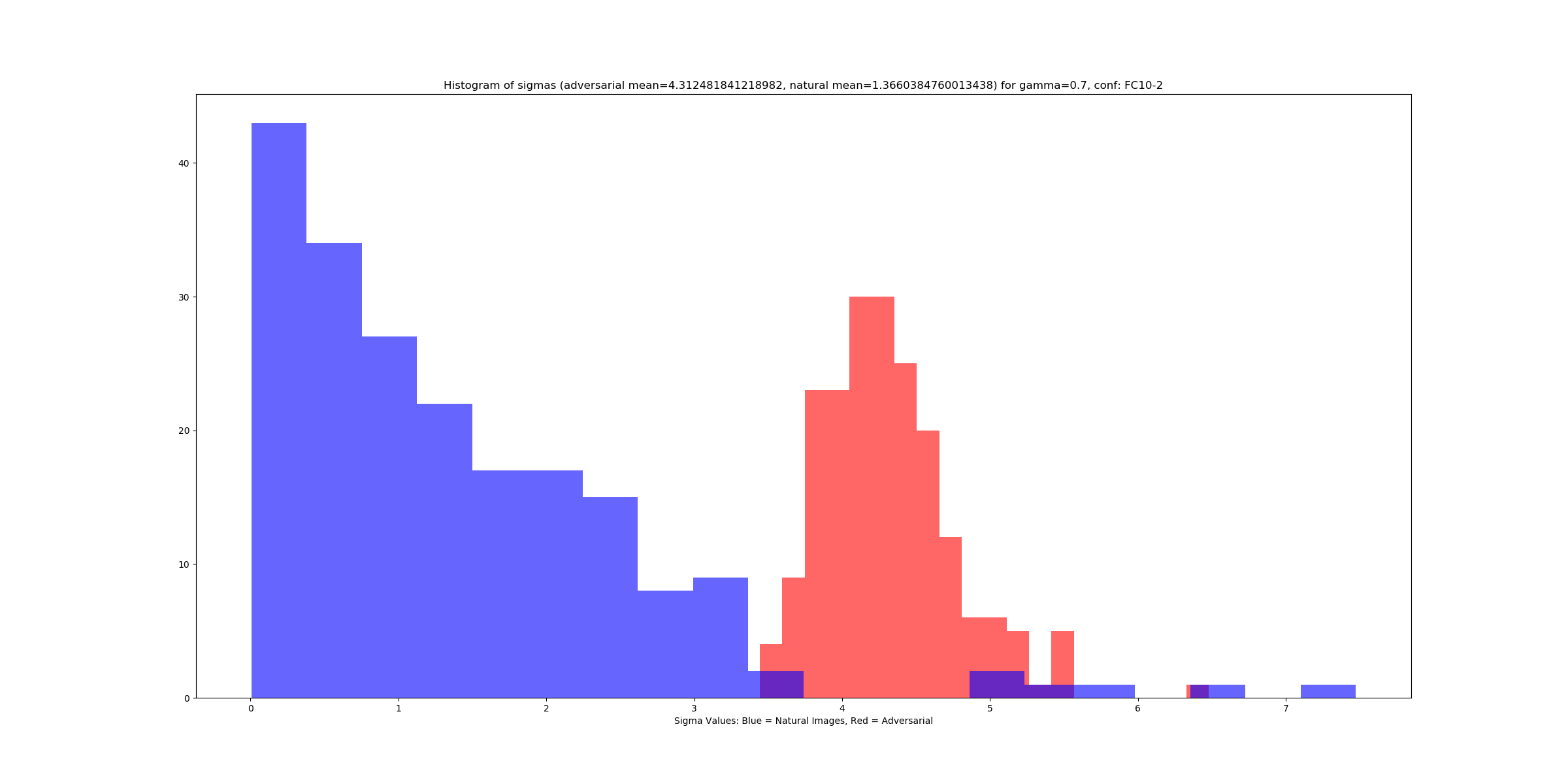}
\includegraphics[trim=200 80 100 100, clip,width=.32\textwidth]{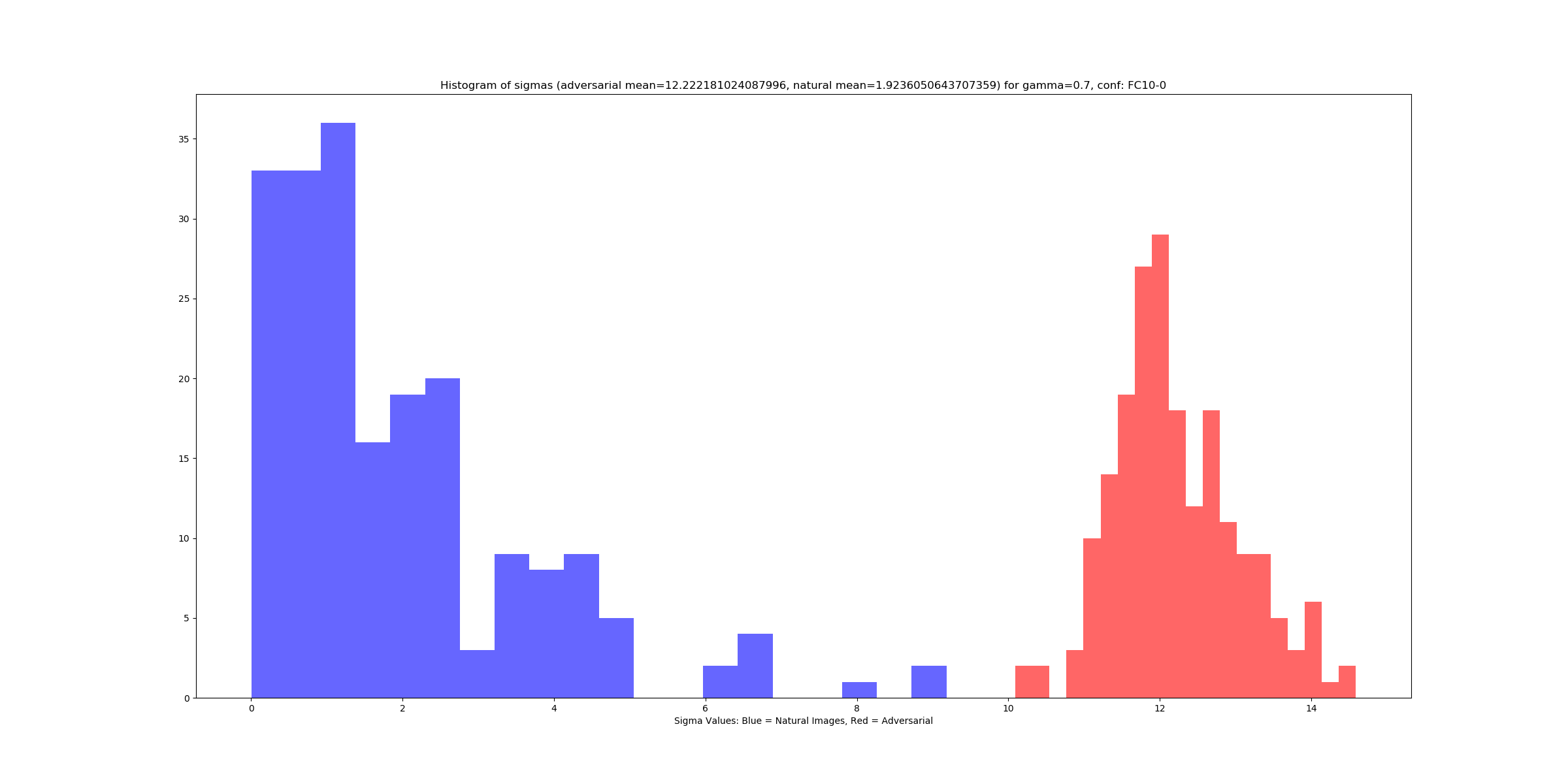}
\caption{Histograms of $0.7$-persistence for FC10-4 (smallest regularization, left), FC10-2 (middle), and FC10-0 (most regularization, right) from Table \ref{table1}. Natural images are in blue, and adversarial images are in red. Note that these are plotted on different scales -- higher regularization forces any "adversaries" to be very stable.\vspace{2em} }
\label{fig:FC10}
\end{figure}

\begin{figure}[!htb]
\centering
\includegraphics[trim=200 80 100 100, clip,width=.49\textwidth]{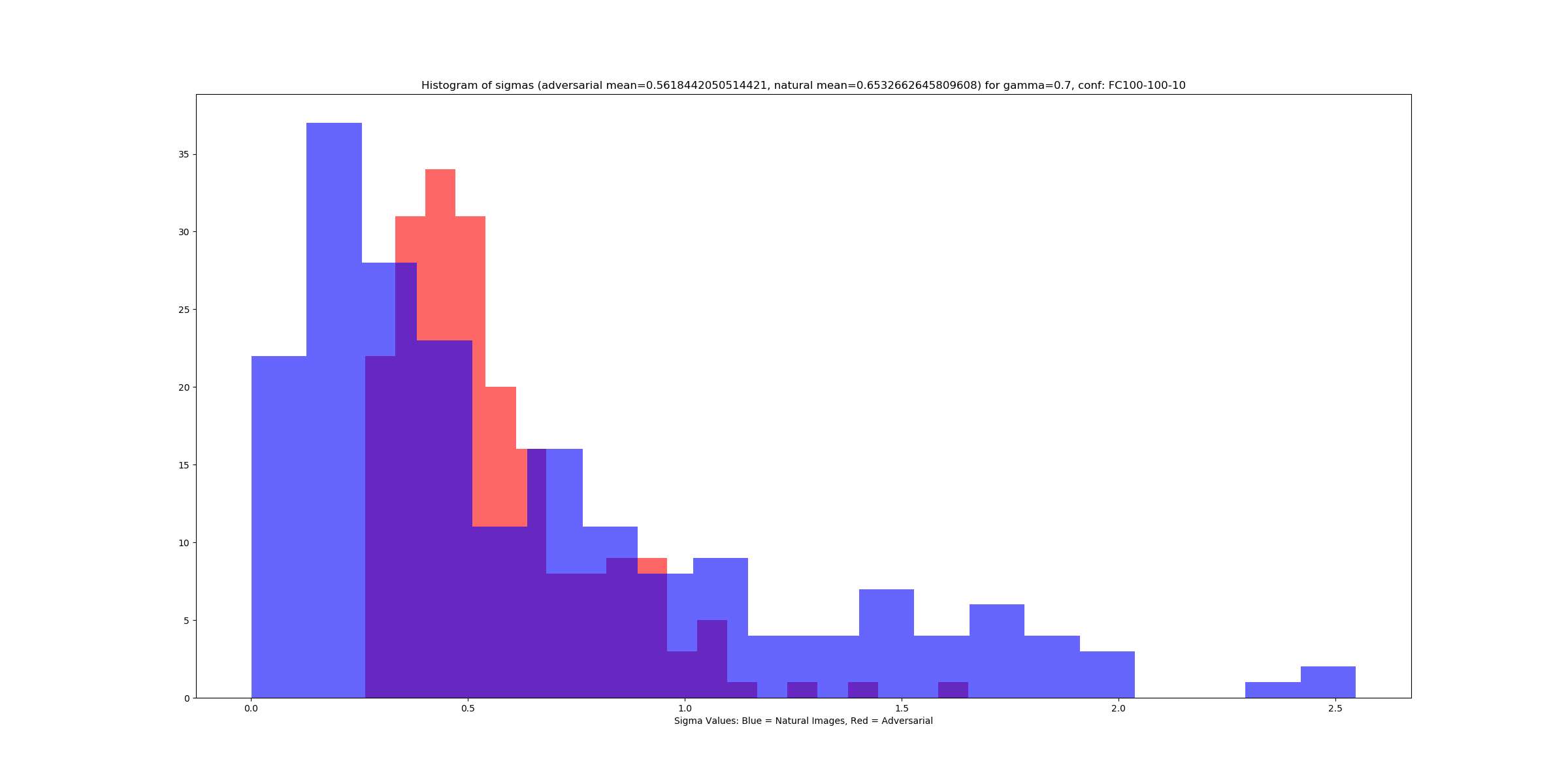}
\includegraphics[trim=200 80 100 100, clip,width=.49\textwidth]{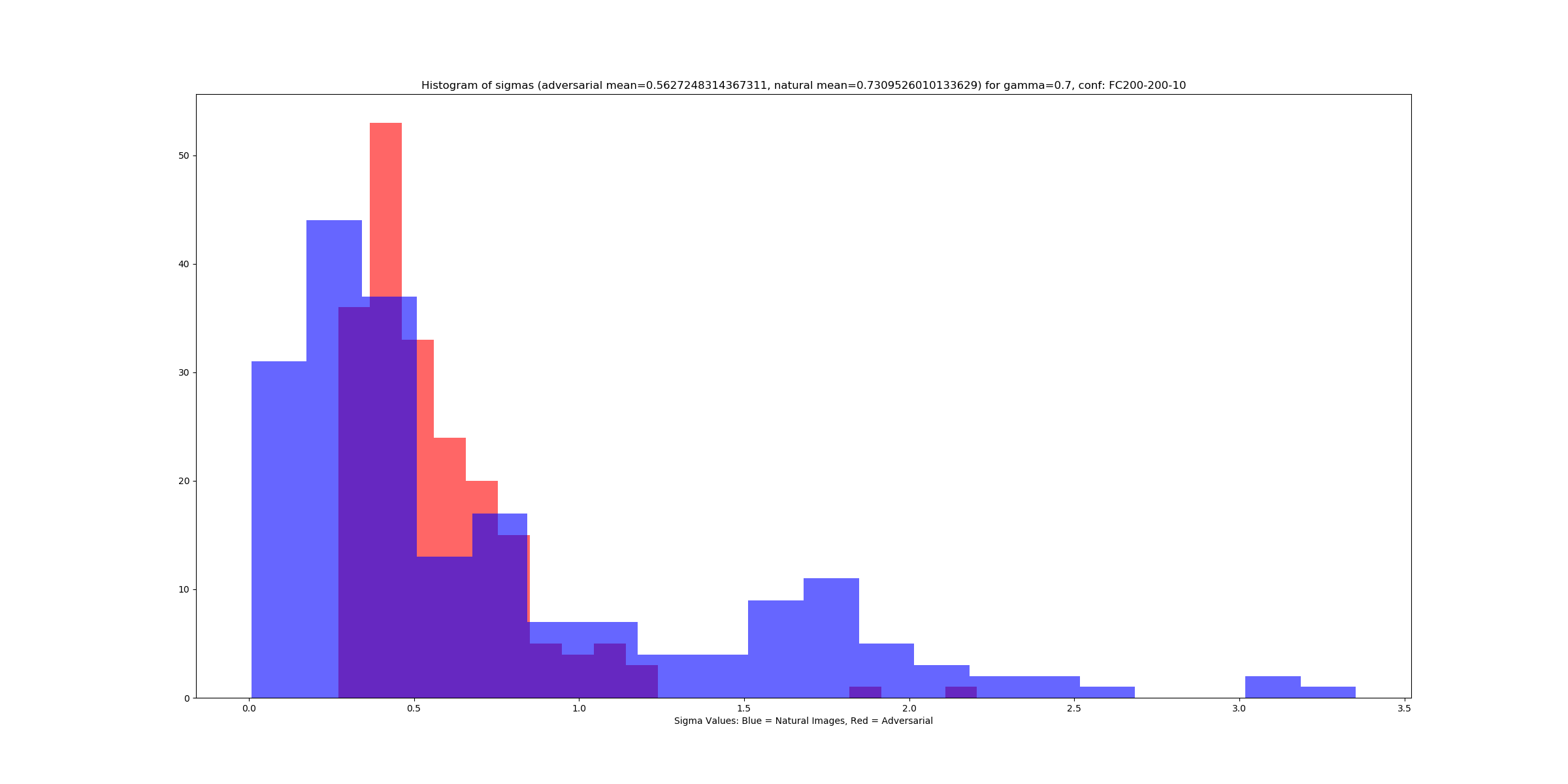}
\caption{Histograms of $0.7$-persistence for FC100-100-10 (left) and FC200-200-10 (right) from Table \ref{table1}. Natural images are in blue, and adversarial images are in red.}
\label{fig:FC100200}
\end{figure}

\begin{figure}[!htb]
\includegraphics[width=.49\textwidth]{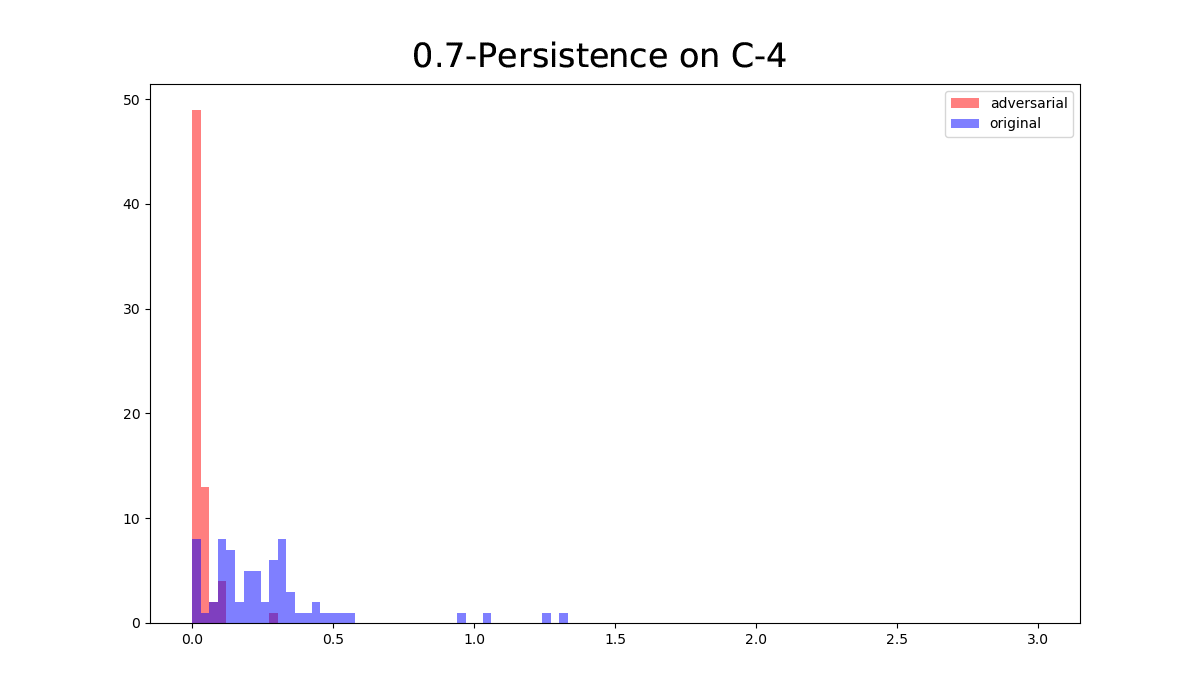}
\includegraphics[width=.49\textwidth]{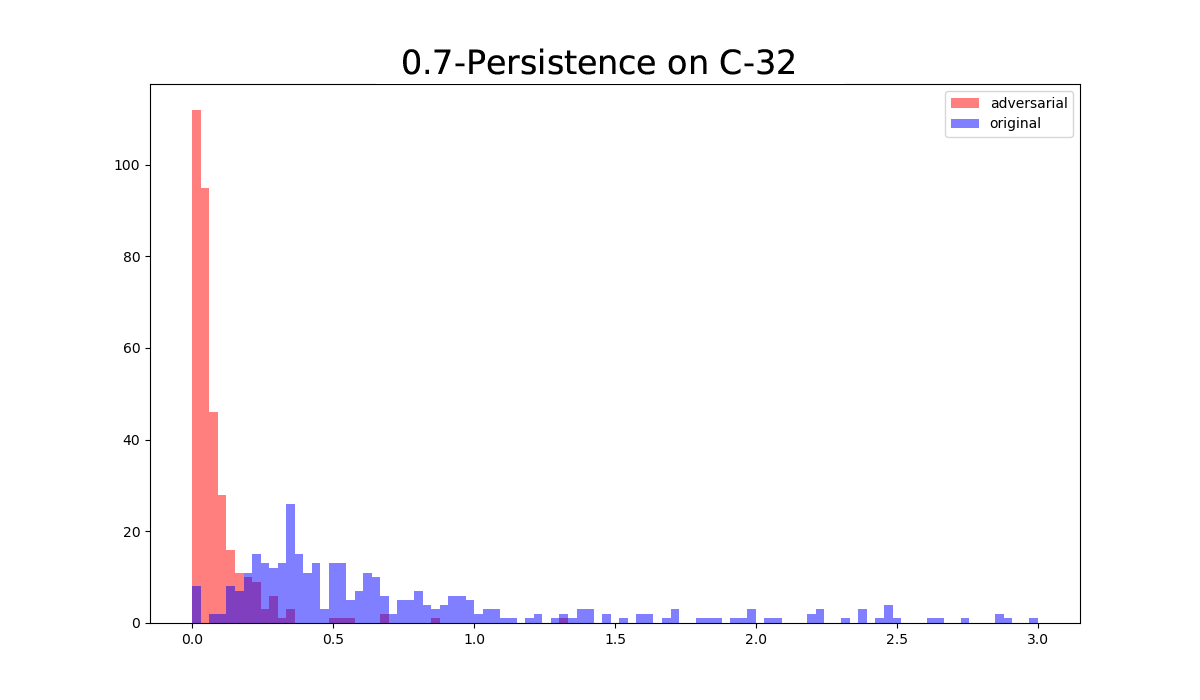}
\includegraphics[width=.49\textwidth]{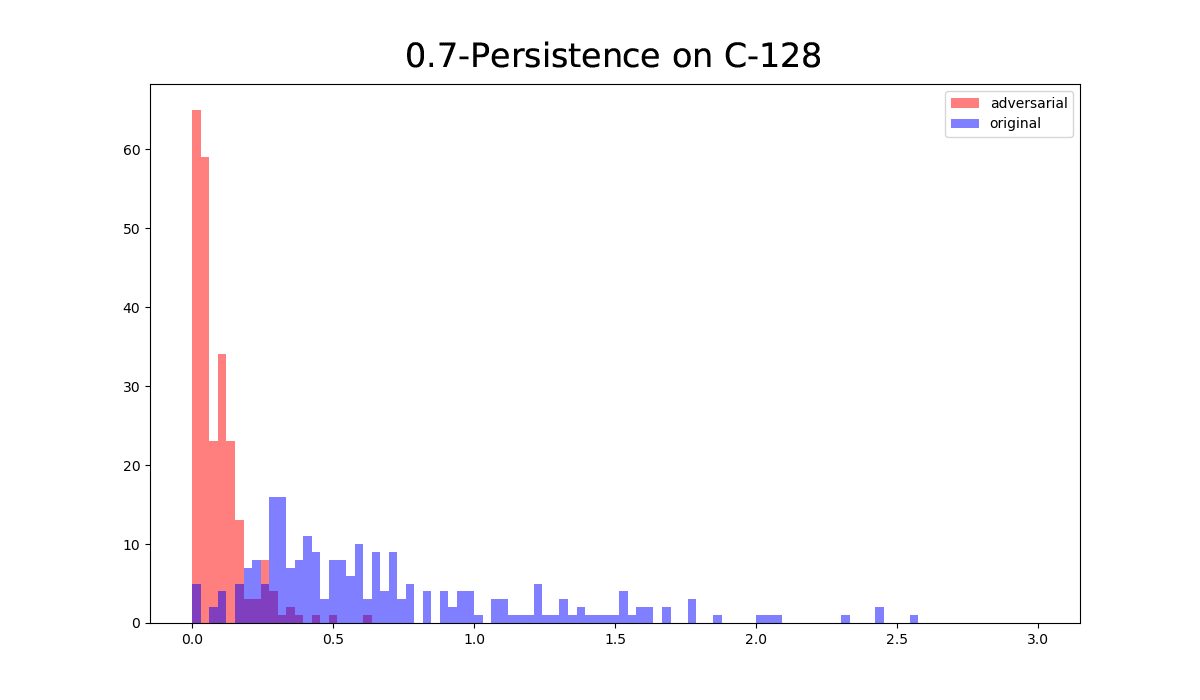}
\includegraphics[width=.49\textwidth]{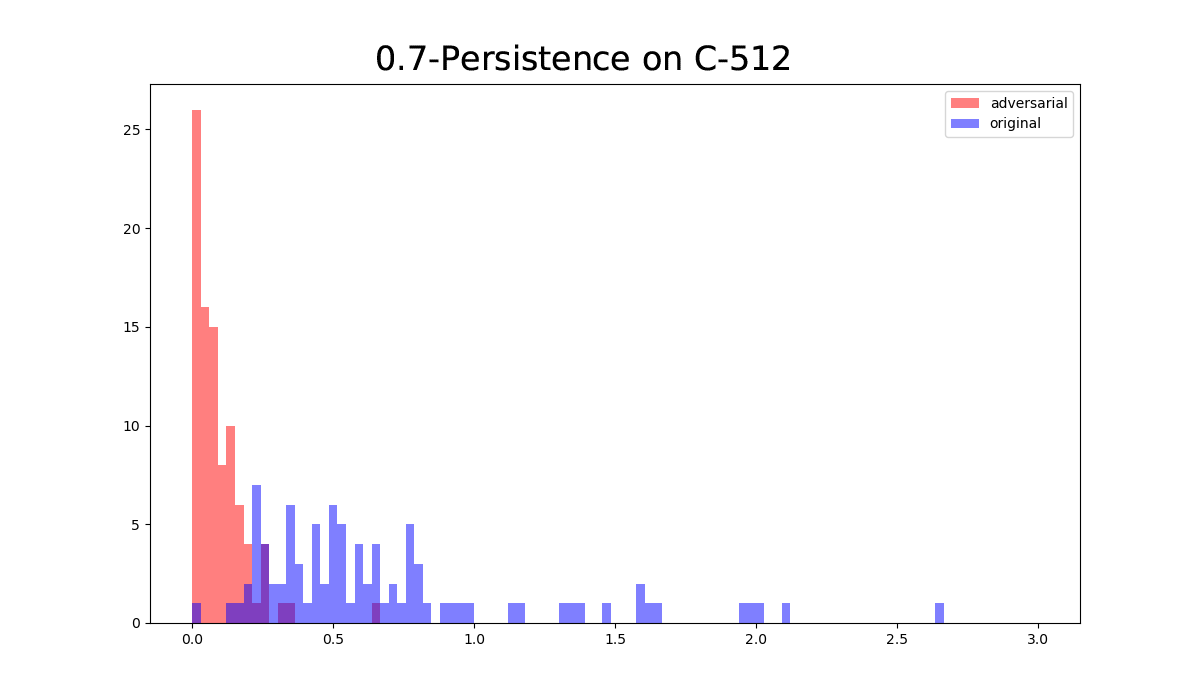}
\caption{Histograms of $0.7$-persistence for C-4 (top left), C-32 (top right), C-128 (bottom left), and C-512 (bottom right) from Table \ref{table1}. Natural images are in blue and adversarial images are in red.}
\label{fig:CNNs}
\end{figure}

\subsection{Additional figures for ImageNet}

In this section we show some additional figures of Gaussian sampling for ImageNet. In Figure \ref{fig:moreimagenet} we see Gaussian sampling of an example of the class \texttt{indigo\_bunting} and the frequency samplings for adversarial attacks of \texttt{goldfinch} toward \texttt{indigo\_bunting} (classifier: alexnet, attack: PGD) and  toward  \texttt{alligator\_lizard} (classifier: vgg16, attack: PGD). Compare the middle image to Figure \ref{fig:imagenet_adv}, which is a similar adversarial attack but used the vgg16 network classifier and the BIM attack. Results are similar. Also note that in each of the cases in Figure \ref{fig:moreimagenet} the label of the original natural image never becomes the most frequent classification when sampling neighborhoods of the adversarial example. 

\begin{figure}[!htb]
    \centering
    \includegraphics[width=.32\textwidth]{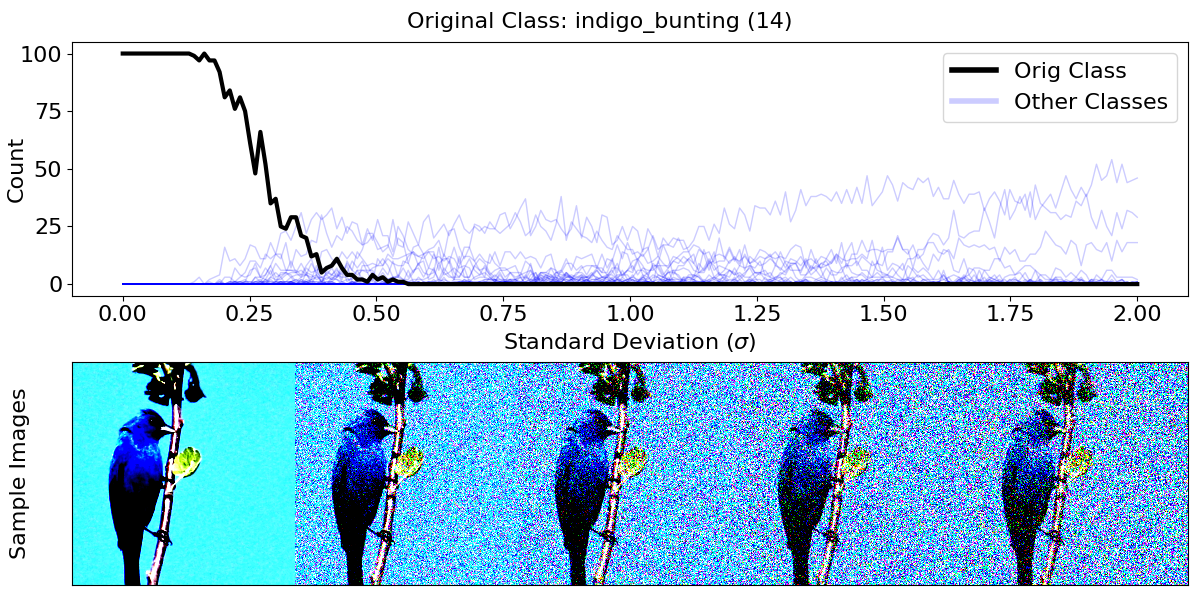}
    \includegraphics[width=.32\textwidth]{figures/IMNET_class_11_alexnet_PGD_48_attack_data_023.png}
    \includegraphics[width=.32\textwidth]{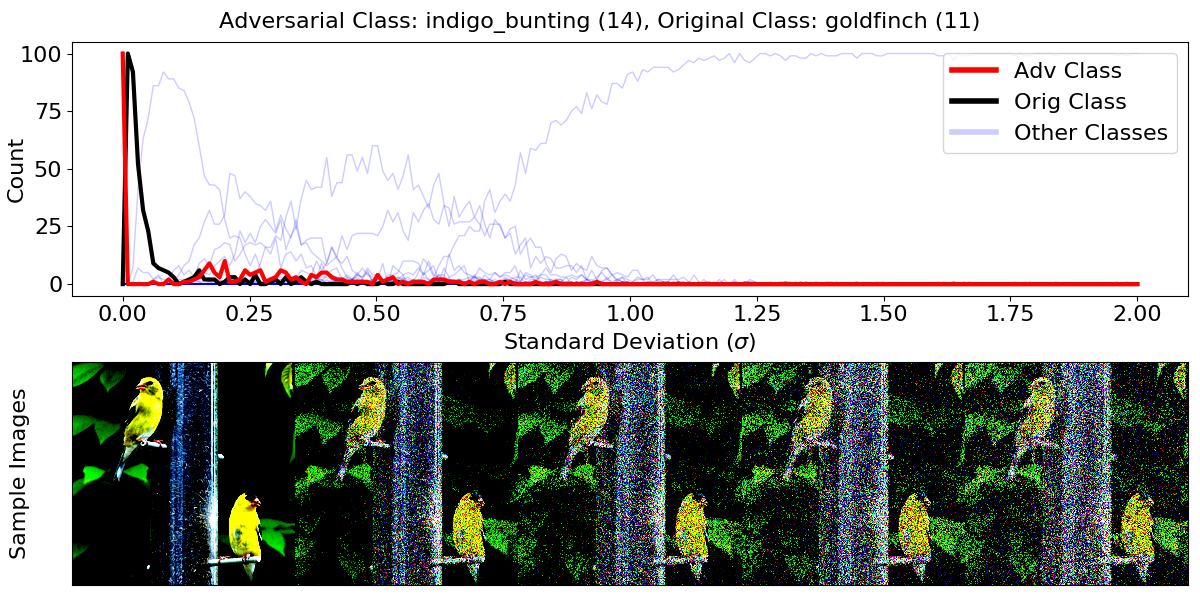}
    \caption{Frequency of each class in Gaussian samples with increasing variance around an \texttt{indigo\_bunting} image (left), an adversarial example of the image in class \texttt{goldfinch} from Figure \ref{fig:imagenet_adv} targeted at the \texttt{indigo\_bunting} class on a alexnet network attacked with PGD (middle), and an adversarial example of the \texttt{goldfinch} image targeted at the \texttt{alligator\_lizard} class on a vgg16 network attacked with PGD (right). Bottoms show example sample images at different standard deviations.}
    \label{fig:moreimagenet}
\end{figure}

In Figure \ref{fig:persistencediffgamma}, we have plotted $\gamma$-persistence along a straight line from a natural image to an adversarial image to it with differing values of the parameter $\gamma$. The $\gamma$-persistence in each case seems to change primarily when crossing the decision boundary. Interestingly, while the choice of $\gamma$ does not make too much of a difference in the left subplot, it leads to more varying persistence values in the right subplot of Figure \ref{fig:persistencediffgamma}.  This suggests that one should be careful not to choose too small of a $\gamma$ value, and that persistence does indeed depend on the landscape of the decision boundary described by the classifier.

\clearpage
\begin{figure}[!htb]
    \centering
    \includegraphics[width=.49\textwidth]{figures/persistence_interpolation_IMNET_class_11_vgg16_BIM_48_attack_data_001.png}
    \includegraphics[width=.49\textwidth]{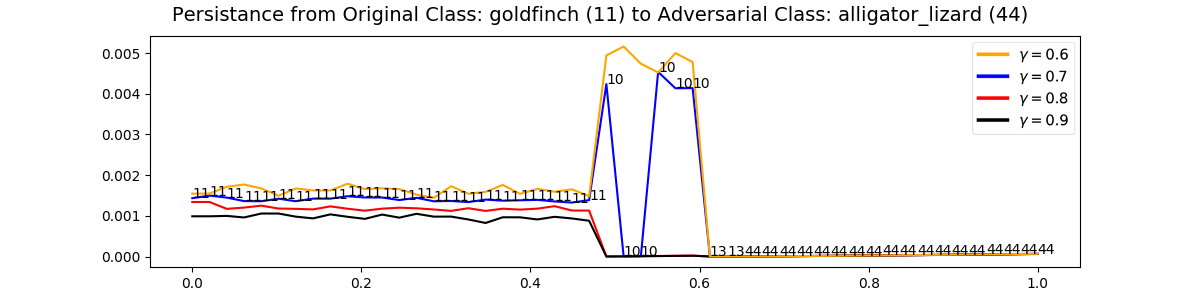}
    \caption{The $\gamma$-persistence of images along the straight line path from an image in class \texttt{goldfinch} (11) to an adversarial image generated with BIM in the class \texttt{indigo\_bunting} (14)  (left) and to an adversarial image generated with PGL in the class \texttt{alligator\_lizard} (44) (right) on a vgg16 classifier with different values of $\gamma$. The classification of each image on the straight line is listed as a number so that it is possible to see the transition from one class to another. The vertical axis is $\gamma$-persistence and the horizontal axis is progress towards the adversarial image.}
    \label{fig:persistencediffgamma}
\end{figure}

\section{Concentration of measures} \label{sec:concentration}

We use Gaussian sampling with varying standard deviation instead of sampling the uniform distributions of balls of varying radius, denoted $U(B_r(0))$ for radius $r$ and center $0$. This is for two reasons. The first is that Gaussian sampling is relatively easy to do. The second is that the concentration phenomenon is different. This can be seen in the following proposition.

\begin{proposition} \label{prop:concentration}
    Suppose $x \sim N(0,\sigma^2 I)$ and $y \sim U(B_r(0))$ where both points come from distributions on $\RR^n$. For $\varepsilon < \sqrt{n}$ and for $\delta < r$ we find the following:
    \begin{align}
        \mathbb{P}\left[\rule{0pt}{15pt} \left| \rule{0pt}{10pt} \Norm{x} - \sigma \sqrt{n} \right| \leq \varepsilon \right] &\geq 1-2e^{-\varepsilon^2/16} \\
        \mathbb{P}\left[\rule{0pt}{15pt} \left| \rule{0pt}{10pt} \Norm{y} - r \right| \leq \delta \right] &\geq 1-e^{-\delta n/r} 
    \end{align}
\end{proposition}
\begin{proof}
    This follows from \cite[Theorems 4.7 and 3.7]{wegner2021lecture}, which are the Gaussian Annulus Theorem and the concentration of measure for the unit ball, when taking account of varying the standard deviation $\sigma$ and radius $r$, respectively.
\end{proof}

The implication is that if we fix the dimension and let $\sigma$ vary, the measures will always be concentrated near spheres of radius $\sigma \sqrt{n}$ and $r$, respectively, in a consistent way. In practice, Gaussian distributions seem to have a bit more spread, as indicated in Figure \ref{fig:sampling}, which shows the norms of $100,000$ points sampled from dimension $n=784$ (left, the dimension of MNIST) and $5,000$ points sampled from dimension $n=196,608$ (right, the dimension of ImageNet).

\begin{figure}[htb]
    \centering
    \includegraphics[width=.5\textwidth]{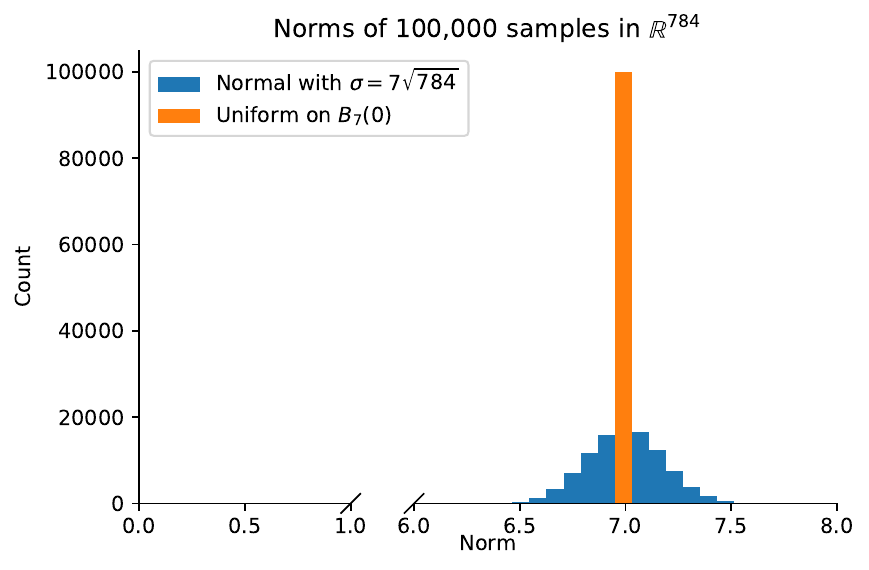}
    \includegraphics[width=.48\textwidth]{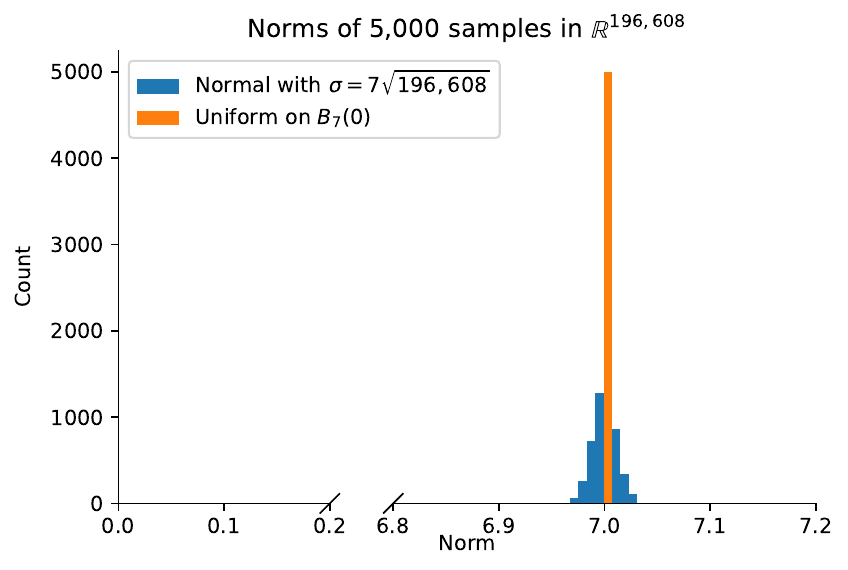}
    \caption{Comparison of the length of samples drawn from $U(B_7(0))$ and $N(0,7\sqrt{n})$ for $n=784$, the dimension of MNIST, (left) and $n=196,608$, the dimension of ImageNet, (right).}
    \label{fig:sampling}
\end{figure}

\section{Licenses of Assets}

We acknowledge the use of the following licensed materiasl: PyTorch (BSD License), MNIST (CC BY-SA 3.0 License), Imagenet (No License -- downloaded in accordance with Princeton and Stanford University terms of access), and TorchAttacks (MIT License).

\end{document}